\documentclass{article}

\usepackage[eandd,final]{neurips_2026}

\usepackage[utf8]{inputenc} 
\usepackage[T1]{fontenc}    
\usepackage{hyperref}       
\usepackage{url}            
\usepackage{booktabs}       
\usepackage{amsmath}
\usepackage{amssymb}
\usepackage{amsfonts}
\usepackage{amsthm}
\usepackage{bm}      
\usepackage{nicefrac}       
\usepackage{microtype}      
\usepackage{xcolor}      
\usepackage[table]{xcolor}
\usepackage{graphicx}
\usepackage{fontawesome5} 
\newtheorem{theorem}{Theorem}[section]

\usepackage{enumitem}
\usepackage{algorithm}
\usepackage{algpseudocode}
\usepackage{adjustbox}
\usepackage{tabularx}
\usepackage{placeins}
\usepackage{multirow}
\definecolor{highlightcell}{RGB}{255,248,220}
\hypersetup{hidelinks}
\usepackage[most]{tcolorbox}

\newcolumntype{C}[1]{>{\centering\arraybackslash}p{#1}}

\newcommand{\highddm}[1]{\cellcolor{red!15}#1}
\title{Your Benchmark Is an Empirical Measure Over Difficulty
}

\author{%
  Yifan Sun\textsuperscript{\dag*}
  \quad~
  Naicheng Yu\textsuperscript{\ddag*}
  \quad~
  Jiawen Gong\textsuperscript{\dag*}
  \quad~
  Jingyan Shen\textsuperscript{\S}
  \quad~
  Huan Zhang\textsuperscript{\dag}
  \\[4pt]
  \textsuperscript{\dag}University of Illinois Urbana-Champaign
  \quad
  \textsuperscript{\ddag}University of California San Diego
  \\
  \textsuperscript{\S}New York University
}

\begin{document}

\maketitle

\vspace{-1.9em}
\begin{center}
  \small
  \href{https://benchmark-difficulty-measure.github.io/}{\faGlobe\enspace\textbf{Project Website:}\enspace\texttt{benchmark-difficulty-measure.github.io}}\\[4pt]
  \href{https://github.com/ASTRAL-Group/benchmark-difficulty-measure}{\faGithub\enspace\textbf{GitHub Code:}\enspace\texttt{https://github.com/ASTRAL-Group/benchmark-difficulty-measure}}\\[4pt]
  \href{https://huggingface.co/datasets/Yifan99/BenchFill}{\raisebox{-0.2em}{\includegraphics[height=1.15em]{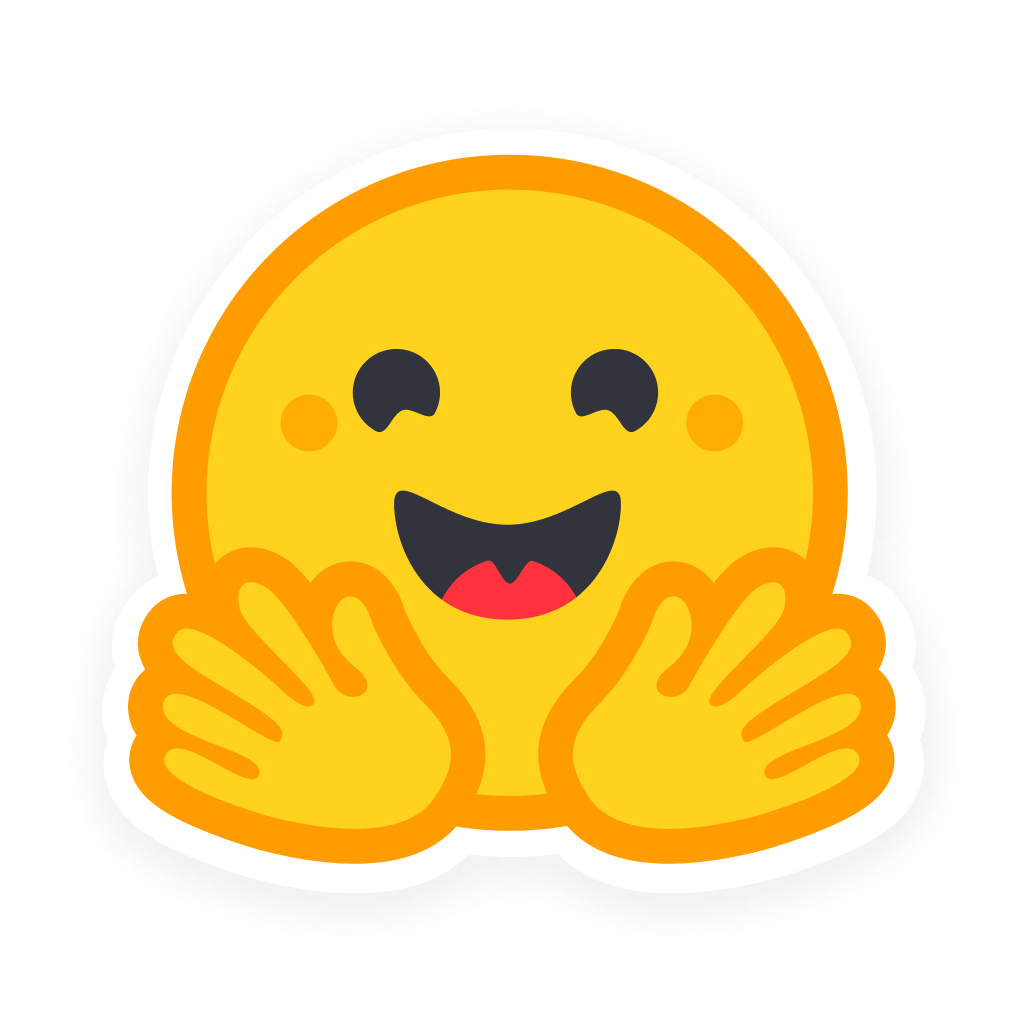}}\enspace\textbf{BenchFill Datasets:}\enspace\texttt{https://huggingface.co/datasets/Yifan99/BenchFill}}
\end{center}
\vspace{0.5em}
\begingroup
\renewcommand{\thefootnote}{\fnsymbol{footnote}}
\footnotetext[1]{%
\begin{tabular}[t]{@{}l@{}}
Equal contribution.\\
Correspondence to Yifan Sun
<\texttt{yifan50@illinois.edu}> and Huan Zhang
<\texttt{huan@huan-zhang.com}>.
\end{tabular}%
}
\endgroup

\begin{abstract}
  A benchmark is not merely a collection of test questions; it naturally induces an empirical measure over difficulty. In this work, we develop this measure-based view of benchmark evaluation. This perspective allows us to identify and formalize two central failure modes: \emph{measure mismatch}, where the benchmark difficulty distribution differs from the intended difficulty distribution, and \emph{coverage gaps}, where uncovered regions along the difficulty axis prevent benchmarks from detecting meaningful localized training progress. We propose two diagnostic metrics that capture these two failure modes and provide rigorous theoretical guarantees. Empirically, we analyze \textbf{25} popular benchmarks spanning \textbf{7} domains using carefully estimated item difficulties. Our results show that many benchmarks implicitly overweight certain difficulty regions, so their leaderboard conclusions should be interpreted with caution; others leave substantial portions of the difficulty axis uncovered, limiting their ability to monitor the model improvement process. Finally, we turn diagnosis into action by introducing \textsc{BenchFill}, a simple difficulty-targeted item rewriting algorithm that effectively reduces local coverage gaps.
\end{abstract}
\section{Introduction}

Benchmark evaluation looks simple: we run models on a set of questions and average the results. The resulting score is then used in two main ways: to compare models on leaderboards~\citep{alzahrani2024benchmarks,wang2018glue,liang2022holistic,myrzakhan2024open}, and to monitor progress during training~\citep{liu2025prorl,hu2025open,he2025skywork,bercovich2025llama,luo2025deepscaler}. But this simplicity hides a basic issue: the conclusion depends not only on how well models perform, but also on how benchmark questions are distributed~\cite{siska2024examining,zhang2026rankllm}.

Consider first the leaderboard setting in Fig.~\ref{fig:teaser}(a), where we want to compare two models. Suppose they perform similarly on easy questions but differ substantially on hard ones. If the benchmark places most of its weight on easy questions, while the difficulty distribution of interest places more weight on harder questions, the two evaluations can lead to different conclusions for the same pair of models: the benchmark may suggest that the two models are nearly tied, even though under the desired difficulty distribution they are clearly separated.  Now consider the training setting in Fig.~\ref{fig:teaser}(b), where we want to track whether a model is improving. Suppose the model improves \textit{substantially}  over a local difficulty range. If the benchmark has few or no questions in that region, the benchmark score may barely change, even though the model has made real progress.

These two failure cases suggest that benchmark conclusions depend critically on how questions are distributed. To study this structure systematically, we organize benchmark questions along a \textbf{difficulty} axis, which provides a natural progression from easier to harder problems. The key conceptual step is to use difficulty as a way to define the distributional structure of the benchmark itself. This view recasts a benchmark from a collection of test questions into an \textbf{empirical measure} over difficulty, with each question contributing equal evaluation mass at its difficulty level.

This measure-theoretic formalization is useful because it turns the two failure cases above into precise and analyzable problems about the benchmark-induced difficulty measure. The first is a \textbf{measure mismatch} problem: the benchmark may distribute its mass differently from the target difficulty distribution that represents the intended evaluation objective, so benchmark conclusions can differ from the conclusions under the target difficulty distribution. The second is a \textbf{coverage gap} problem. Because the benchmark is an empirical measure supported only on finitely many difficulty points, there may be intervals of difficulty with little or no benchmark mass. A substantial performance gain concentrated in such an interval may therefore produce little or no change in the benchmark score. To our knowledge, this is the first work to identify and formalize measure mismatch and coverage gaps as two fundamental failure modes of benchmark evaluation.

These two problems naturally lead to two simple diagnostic metrics: the Wasserstein distance~\cite{villani2009optimal} between the benchmark measure and a target difficulty distribution, and the maximum gap between adjacent benchmark difficulty points. We provide theoretical results showing that these quantities directly characterize distributional mismatch and the detectability of localized improvements. Empirically, we analyze \textbf{25} popular benchmarks spanning \textbf{7} domains. We estimate question difficulty using Item Response Theory and compute the proposed metrics. Our results show that both issues are prevalent across benchmarks.

\textbf{Our key contributions are summarized as follows:}
\vspace{0.2cm}
\begin{itemize}[wide, labelwidth=!, noitemsep, labelindent=0pt,topsep=-\parskip]
    \item We introduce a new measure-based perspective on benchmark evaluation, representing a benchmark as an empirical measure over difficulty. This perspective formalizes measure mismatch and coverage gaps, motivates two diagnostic metrics, and yields rigorous theoretical guarantees.

    \item We conduct an empirical study across \textbf{25} benchmarks spanning \textbf{7} domains, showing that both difficulty-distribution mismatch and large coverage gaps are widespread in current benchmarks.

    \item We propose \textsc{BenchFill}, a difficulty-targeted gap-filling algorithm. With only three inserted items per benchmark, \textsc{BenchFill} reduces the sum of the top three difficulty gaps by $21.09\%$--$38.88\%$.
\end{itemize}

\begin{figure}[t]
    \centering
    \resizebox{\textwidth}{!}{\includegraphics{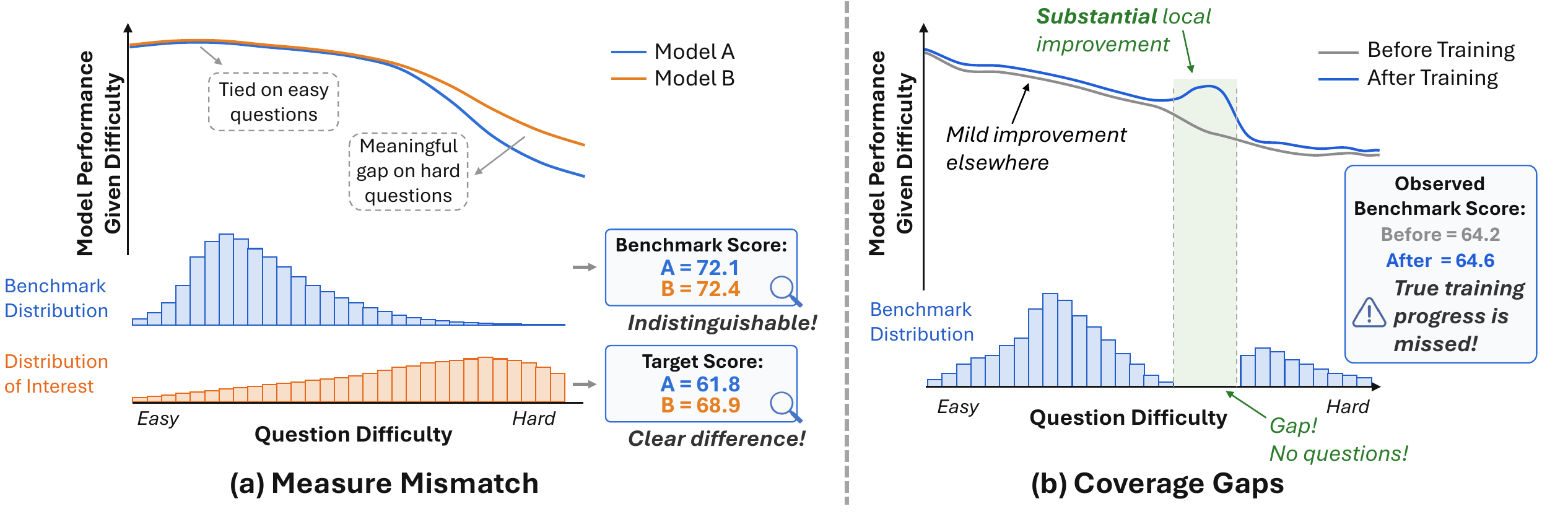}}
      \vspace{-4mm}
    \caption{\textbf{Two benchmark failure cases.} In both panels, the curves show true model \textit{conditional} performance \textit{as a function of question difficulty}.
    (a) \textbf{Measure Mismatch.}  
    The benchmark emphasizes easy questions, where the models perform similarly, while the difficulty distribution we care about places more weight on harder questions, where they differ substantially.
    Because the benchmark and the target distribution weight the difficulty axis differently, they lead to different conclusions for the same two models.
    (b) \textbf{Coverage Gaps.} A model actually improves \textit{substantially} over a localized difficulty region. Because the benchmark has no questions in that region, this improvement is not reflected in the observed benchmark score, causing meaningful training progress to be missed.}
    \label{fig:teaser}
  \vspace{-6mm}
\end{figure}
\section{Formalization and Theory}

\paragraph{Difficulty space and model performance function.} 
We represent the difficulty of a benchmark item, e.g., a question, by a continuous variable $z\in[0,1]$, where larger values correspond to harder items. For a model $\theta$, let $s_\theta(z)\in[0,1]$ denote its conditional expected performance on items with difficulty $z$, e.g., pass@1.
  \vspace{-2mm}
\paragraph{Benchmark distribution and score.}
Let a benchmark $B=\{x_i\}_{i=1}^n$ consist of $n$ items, and let
$z_i\in[0,1]$ denote the difficulty of item $x_i$. The benchmark score is $S_B(\theta):=\frac{1}{n}\sum_{i=1}^n s_\theta(z_i)$. This average implicitly induces an empirical measure over the difficulty axis,
\[
\nu_B=\frac{1}{n}\sum_{i=1}^n \delta_{z_i},
\]
where $\delta_z$ denotes a point mass at difficulty $z$. Under this representation, the benchmark score is exactly an integral with respect to the benchmark-induced difficulty measure, $S_B(\theta)=\int s_\theta(z)\,d\nu_B(z)$.

\paragraph{Performance difference between two models.}
For two models $\theta$ and $\theta'$, define the \emph{difficulty-conditioned performance difference} as $\Delta(z):=s_{\theta'}(z)-s_\theta(z)$, where the dependence of $\Delta$ on $(\theta',\theta)$ is omitted for notational simplicity. The performance difference \textit{observed on benchmark $B$} is
\[
D_B(\theta',\theta)
:= S_B(\theta')-S_B(\theta)
= \frac{1}{n}\sum_{i=1}^n \Delta(z_i)
= \int \Delta(z)\,d\nu_B(z).
\]

\subsection{Measure Mismatch}
\paragraph{Benchmark-induced coupling.}
The quantity $D_B(\theta',\theta)$ implicitly couples two distinct objects: the \textit{model-dependent function} $\Delta(z)$ and the \textit{benchmark-dependent measure} $\nu_B$. Thus, a leaderboard comparison is not merely determined by the difficulty-conditioned performance difference between the models; it also depends on how the benchmark allocates mass over the difficulty axis. If this allocation differs from the difficulty distribution of interest, the benchmark aggregates the same model difference under a mismatched measure. This mismatch can distort the observed model gap and, in severe cases, change the induced ordering.

\paragraph{Target distribution.}
Formally, we introduce a target distribution
$
\mu \in \mathcal{P}([0,1]),
$
which represents the intended evaluation distribution over the difficulty space. This target distribution may encode the desired application mix. Under $\mu$, the target score of model $\theta$ and the target performance difference between $\theta'$ and $\theta$ are defined as
\[
S_\mu(\theta)
:=
\int s_\theta(z)\, d\mu(z),
\qquad
D_\mu(\theta',\theta)
:=
S_\mu(\theta') - S_\mu(\theta)
=
\int \Delta(z)\, d\mu(z).
\]

The discrepancy between $D_B(\theta',\theta)$ and $D_\mu(\theta',\theta)$ arises from the mismatch between the benchmark measure $\nu_B$ and the target distribution $\mu$. This viewpoint allows us to study distortions in leaderboard comparisons as a distribution mismatch problem.

\paragraph{A metric for distribution mismatch.}
We measure how well the difficulty distribution of a benchmark matches a target difficulty distribution using the 1-Wasserstein distance. Specifically, for a benchmark \(B\) with empirical difficulty measure \(\nu_B\) and a target distribution \(\mu\) on \([0,1]\), we define its \emph{Difficulty Distribution Mismatch (DDM)} as
\[
\mathrm{DDM}(B;\mu) := W_1(\nu_B,\mu)
=
\inf_{\pi \in \Pi(\nu_B,\mu)}
\int_{[0,1]^2} |u-v|\, d\pi(u,v),
\]
where $\Pi(\nu_B,\mu)$ denotes the set of all couplings of $\nu_B$ and $\mu$. Intuitively, $W_1(\nu_B,\mu)$ measures how much probability mass must be transported, and how far, to transform the benchmark difficulty distribution into the target one. The following theorem shows that this mismatch directly controls the distortion between benchmark-based and target-based performance differences.

\begin{theorem}[Distribution mismatch controls score distortion and rank stability]
\label{thm:w1_rank_stability}
Suppose the performance difference function $\Delta$ is $L$-Lipschitz. Then
\[
\bigl|D_B(\theta',\theta) - D_\mu(\theta',\theta)\bigr|
\le
L\, W_1(\nu_B,\mu).
\]
Consequently, if $|D_\mu(\theta',\theta)| > L\, W_1(\nu_B,\mu)$, then
\[
\operatorname{sign}\bigl(D_B(\theta',\theta)\bigr)
=
\operatorname{sign}\bigl(D_\mu(\theta',\theta)\bigr),
\]
so the benchmark-induced ordering agrees with the target ordering, i.e., no rank flip can occur.
\end{theorem}

Theorem~\ref{thm:w1_rank_stability} shows that $W_1(\nu_B,\mu)$ bounds how much the benchmark-induced score difference can differ from the target score difference. A small $W_1(\nu_B,\mu)$ means that the benchmark provides a faithful approximation to the target evaluation, while a large $W_1(\nu_B,\mu)$ indicates that the benchmark may overemphasize some difficulty regions and underrepresent others, potentially distorting model comparisons and even changing their induced ranking. App.~\ref{sec:app_proof_mismatch} provides the full proof and extends this result to the finite-sample setting.

\paragraph{A principled target distribution without prior information.}
When prior information is available, it is natural to choose $\mu$ to reflect the intended evaluation objective over the difficulty space. In many settings, however, such prior knowledge is unavailable. In that case, choosing a non-uniform target would implicitly privilege some difficulty regions over others without evidence. To characterize a principled default choice, we consider an agnostic setting in which the location of the performance difference along the difficulty axis is completely unknown. We consider the class
\[
\mathcal{F}_{\rho}
:=
\left\{
\Delta:[0,1]\to[-1,1]\ \text{measurable} :
\int_0^1 \Delta(z)\,dz = \rho
\right\},
\qquad \rho\in(-1,1).
\]
This class encodes an uncertainty set: the signed performance difference has a fixed average level $\rho$ over the difficulty axis, but its location is otherwise unrestricted. 

For any candidate target distribution $\mu \in \mathcal P([0,1])$, define its worst-case retained difference by
\[
V_{\rho}(\mu)
:=
\inf_{\Delta \in \mathcal{F}_{\rho}}
\int \Delta(z)\, d\mu(z).
\]
The following theorem shows that, under complete uncertainty about where along the difficulty axis the performance difference occurs, the uniform distribution is the \textit{unique} maximin target, and is the only choice whose worst-case retained difference remains equal to the average difference $\rho$.

\begin{theorem}[Uniform target as the unique maximin choice]
\label{thm:uniform_maximin}
For $\rho \in (-1,1)$, we have
\[
\sup_{\mu \in \mathcal P([0,1])} V_{\rho}(\mu) = \rho,
\]
and the unique maximizer is the uniform distribution $\lambda=\mathrm{Unif}[0,1]$.
\end{theorem}
Theorem~\ref{thm:uniform_maximin} does not claim that the uniform distribution is universally optimal. Rather, it justifies the uniform target as a principled default when no prior information over the difficulty space is available. In that case, adopting the uniform target avoids introducing an additional preference over difficulty regions and provides a neutrality criterion for benchmark evaluation: the evaluation does not privilege any difficulty region a priori. Full proof is provided in App.~\ref{sec:app_proof_mismatch}.

\subsection{Coverage Gaps}

\paragraph{Discrete benchmark measures may fail to monitor local training improvements.}
Benchmarks are often used not only for final model comparison, but also to monitor training progress~\citep{he2025skywork,yu2025dapo,luo2025deepscaler}. However, training improvements are often local rather than uniform: a model may improve substantially over a particular difficulty range while improving only modestly elsewhere.

Recall that a benchmark score is an integral with respect to the benchmark empirical measure $\nu_B$. Since $\nu_B$ is a \textit{discrete} measure supported only on the finitely many benchmark difficulty points, it assigns zero mass to every interval that contains no benchmark item. Consequently, even if the model improves substantially on a nontrivial contiguous difficulty region, the benchmark score may not reflect this improvement when the region contains no benchmark item.  In this sense, a benchmark can fail as a progress monitor when its local resolution is insufficient.



\paragraph{Maximum gap.}
We first characterize the local resolution of a benchmark through its \emph{maximum gap}. Let $0 \le z_{(1)} \le \cdots \le z_{(n)} \le 1$ be the ordered benchmark difficulty points, and define $z_{(0)}:=0$ and $z_{(n+1)}:=1$. The \emph{maximum gap} of benchmark $B$ is
\begin{equation}
G(B):=\max_{0\le j\le n}\bigl(z_{(j+1)}-z_{(j)}\bigr)
\label{eq:max_gap}
\end{equation}
i.e., the length of the largest interval on the difficulty axis that contains no benchmark point.

\paragraph{A hypothesis testing formulation.}
To formalize whether a benchmark can reliably detect localized training improvements, we cast the problem as a hypothesis test. Under the null hypothesis, there is no improvement. Under the alternative, the post-improvement performance exceeds the pre-improvement performance by at least $a$ over some contiguous difficulty interval of width $w$. The following theorem shows that the benchmark maximum gap $G(B)$ is the exact asymptotic minimax detection boundary.

\begin{theorem}[Maximum gap as an exact asymptotic minimax detection boundary]
\label{thm:max-gap-detection-boundary}
Fix a test level $\alpha\in(0,1)$ and $a\in(0,1]$. Let $s^-,s^+:[0,1]\to[0,1]$ denote the pre- and post-improvement performance curves, respectively. For a benchmark $B$, let $\Pi_\alpha^*(B;w,a,k)$ denote the minimax worst-case power of level-$\alpha$ tests based on $k$ repeated samples per benchmark item for testing
\[
H_0:\ s^-=s^+
\qquad\text{versus}\qquad
H_1:\ \exists\ \text{interval } I\subset[0,1],\ |I|=w,\  s^+(z)-s^-(z)\ge a\ \ \forall z\in I.
\]
\begin{enumerate}
    \item If $w \le G(B)$, then $\Pi_\alpha^*(B;w,a,k)=\alpha$ for every $k\in\mathbb N$.
    \item If $w > G(B)$, then $\lim_{k\to\infty}\Pi_\alpha^*(B;w,a,k)=1$.
\end{enumerate}
Therefore, $G(B)$ exactly characterizes the asymptotic minimax detection boundary for benchmark $B$.
\end{theorem}

Theorem~\ref{thm:max-gap-detection-boundary} shows that the discrete geometry of the benchmark induces a sharp detectability phase transition. If the width of a local improvement does not exceed the maximum adjacent difficulty gap, then no level-$\alpha$ test can achieve nontrivial worst-case power, regardless of how many repeated evaluations are performed. In other words, such an improvement can be completely invisible to the benchmark in the worst case. Once the improvement width exceeds the maximum gap, repeated evaluation can eventually make the signal detectable. In this sense, $G(B)$ characterizes the intrinsic local resolution limit of benchmark $B$. 
The full proof is given in App.~\ref{sec:app_proof_missed}.
\section{Empirical Diagnosis of Current Benchmarks}

\begin{figure}[t]
    \centering
    \resizebox{\textwidth}{!}{\includegraphics{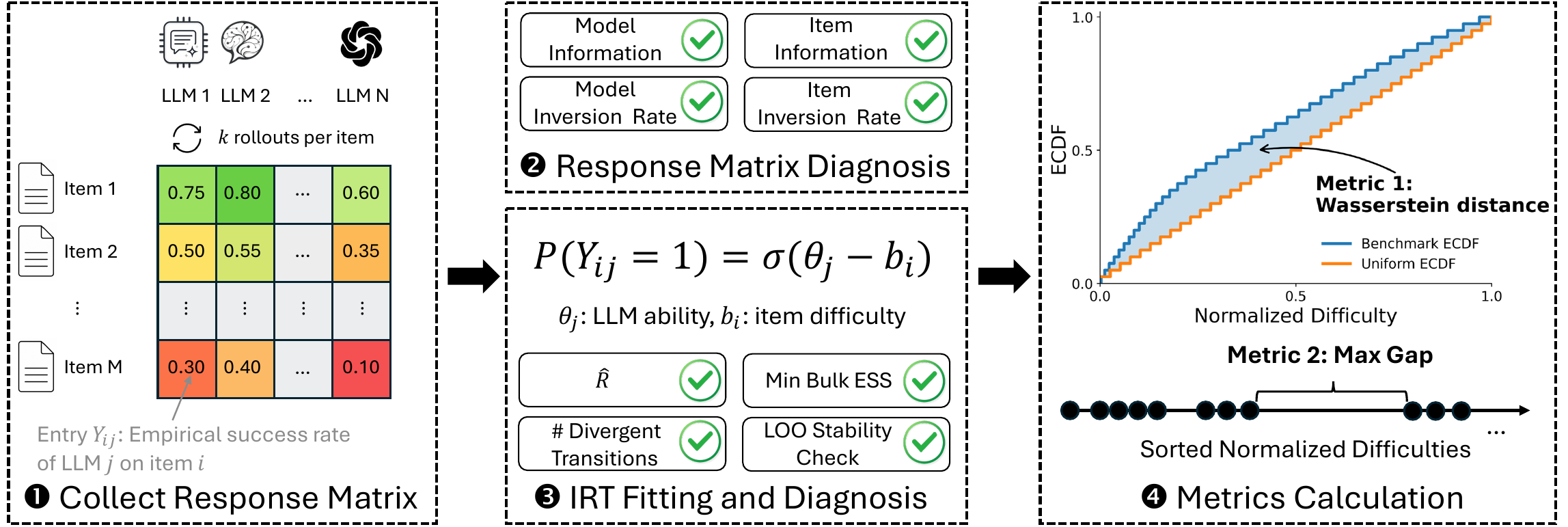}}
      \vspace{-4mm}
    \caption{\textbf{Overview of our benchmark difficulty estimation and metric calculation pipeline.} We collect empirical success rates from multiple LLMs on benchmark items, diagnose the response matrix, fit a 1PL IRT model to estimate item difficulties, perform post-fit IRT diagnostics, and compute two metrics over the normalized difficulty distribution.
    }
    \label{fig:pipeline}
  \vspace{-4mm}
\end{figure}
We analyze a collection of widely used benchmarks from diverse domains using the two metrics proposed earlier. Fig.~\ref{fig:pipeline} provides an overview of our difficulty estimation and metric calculation pipeline.  Specifically, we consider \textbf{25} popular benchmarks spanning math, coding, agentic tasks, logic, science, instruction-following, and medicine. Detailed information about these benchmarks is provided in App.~\ref{sec:app_benchmark_details}.

\subsection{Difficulty Estimation and Metrics Calculation}

\paragraph{Method: Item Response Theory.}
To estimate benchmark item difficulty, one may be tempted to use each item's average failure rate across models~\citep{yu2025dapo,sun2026improving,wen2025light}. However, item difficulty is inherently relative to model ability: it should depend not only on how many models fail an item, but also on how capable the models that fail or solve it are. 
We therefore use Item Response Theory (IRT), an established psychometric framework~\citep{lord1952theory, rasch1993probabilistic} that has been widely adopted for difficulty estimation in LLM benchmark analysis~\citep{ding2024easy2hard, rodriguez2021evaluation, hofmann2025fluid, kipnis2024metabench, polo2024tinybenchmarks, zhou2026lost}. IRT models the probability that a model solves an item as a function of both model ability and item difficulty. In the commonly used one-parameter logistic model (1PL), each model $j$ is assigned an ability parameter $\theta_j$, and each item $i$ is assigned a difficulty parameter $b_i$. The probability that model $j$ correctly solves item $i$ is modeled as
\[
\Pr(Y_{ij}=1)=\sigma(\theta_j-b_i),
\]
where $\sigma(\cdot)$ is the logistic function. Intuitively, items with larger $b_i$ require higher model ability to be solved correctly, and the estimated $b_i$ therefore provides a natural measure of item difficulty. For multiple-choice benchmarks, the model above can be naturally extended with a guessing parameter.

In practice, difficulty estimation begins by collecting the performance of multiple models on benchmark items, forming a response matrix $Y \in \mathbb{R}^{n \times m}$, where $Y_{ij}$ denotes the empirical success rate of model $j$ on item $i$. The latent abilities $\{\theta_j\}_{j=1}^m$ and item difficulties $\{b_i\}_{i=1}^n$ are then estimated by fitting the likelihood induced by the IRT model. 
To ensure reliable difficulty estimation, we take several practical measures, including \textbf{careful model selection} and \textbf{rigorous IRT fitting diagnostics}.

\paragraph{Model selection.}
Accurate difficulty estimation requires models with diverse capability levels. To ensure sufficient ability coverage, we deliberately select models spanning a wide range of capabilities. For non-agentic benchmarks, we select $12$ LLMs, ranging from $1.7$B to $397$B parameters, using the public leaderboard from Text Arena\footnote{\url{https://arena.ai/leaderboard}}. Models are first ordered by their overall capability aggregated across multiple domains, then divided into several approximately equal intervals along this ranking. We perform stratified sampling by selecting one model from each interval and additionally include a few models from both the top and bottom of the ranking to better capture the extremes of the capability spectrum. Each selected model is evaluated on each benchmark item with up to $32$ rollouts using the model-recommended temperature and top-$p$. For agentic benchmarks, tasks typically require strong proprietary models. We therefore rely on the evaluation results from HAL~\cite{hal}, which reports the performance of $15$--$66$ state-of-the-art closed-source LLMs on agentic task benchmarks. Detailed information about the models is provided in App.~\ref{sec:app_model_details}.

\paragraph{IRT fitting and diagnostics.}
We estimate the IRT parameters using Bayesian inference via MCMC. To ensure reliable difficulty estimation, we perform a series of diagnostics both \textit{before} and \textit{after} IRT fitting. Before fitting, we examine the informativeness and structural consistency of the response matrix using $4$ metrics. After fitting, we assess convergence and sampling quality using $3$ standard Bayesian diagnostics, including the Gelman--Rubin statistic ($\hat{R}$), effective sample size (ESS), and the number of divergent transitions. All benchmarks pass all $7$ diagnostic checks.  
For non-agentic benchmarks, we further conduct a leave-one-out (LOO) robustness analysis. Results suggest that our choice of $12$ models is sufficient for stable difficulty estimation. To support future work, we also publicly release the carefully curated difficulty estimates for the $25$ benchmarks\footnote{\url{https://github.com/ASTRAL-Group/benchmark-difficulty-measure/tree/main/data/benchmarks}}. Detailed definitions and diagnostic results are provided in App.~\ref{sec:app_irt}. We also provide additional model-comparison analyses against 2PL and multidimensional IRT models in App.~\ref{app:irt_model_comparison}. 



\paragraph{Proposed metric calculation.}
For each benchmark, after estimating item difficulties, we first normalize them to $[0,1]$ and compute their empirical cumulative distribution function (ECDF). Since in general no additional prior information over the difficulty space is available, we quantify distribution mismatch by computing the 1-Wasserstein distance between the benchmark empirical measure and a same-size empirical \textit{uniform reference} on $[0,1]$. In one dimension, this is equivalently the $L_1$ distance between their ECDFs. We compute the maximum gap according to Eq.~\eqref{eq:max_gap} and also report the sum of the five largest adjacent gaps as a reference metric. We provide posterior uncertainty analyses for both metrics in App.~\ref{app:metric_uncertainty}.

\subsection{Results and Analysis} 

\paragraph{Several benchmarks exhibit large deviations from the neutral uniform target.} Tab.~\ref{tab:main_result} shows that 10 out of 25 benchmarks have $\mathrm{DDM}\ge 0.10$, indicating that their empirical difficulty distributions must be transported by an average distance of at least 10\% of the full difficulty range to match the same-size uniform reference. These benchmarks therefore encode a strong implicit difficulty weighting: their aggregate scores are more sensitive to overrepresented difficulty regions and less sensitive to underrepresented ones. Non-uniformity is not inherently problematic when it matches the intended evaluation objective. However, when no such objective is specified, high-DDM leaderboard results should not be interpreted as difficulty-neutral comparisons. Instead, they reflect comparisons under the benchmark's empirical difficulty weighting, so users should check whether this weighting matches the difficulty distribution relevant to their use case. App.~\ref{sec:app_benchmark_cal} presents a simple benchmark calibration algorithm for moving a benchmark distribution closer to a chosen target.

\paragraph{Many benchmarks exhibit large local coverage gaps.}
Tab.~\ref{tab:main_result} shows that 8 out of 25 benchmarks have $\mathrm{Max\ Gap}\ge 0.10$. As a result, a training run may make meaningful local progress in an uncovered difficulty region, while the aggregate benchmark score remains nearly unchanged. This can cause useful training progress to be underestimated, promising runs to be prematurely discarded, and compute to be misallocated. Moreover, the problem is not limited to the single largest gap. We find that 9 out of 25 benchmarks have $\mathrm{Top\text{-}5\ Sum}\ge 0.30$. For example, Fig.~\ref{fig:main_result} shows that AIME24 and AIME25 have top-5 gap sums close to $0.5$, leaving large portions of the difficulty axis locally unresolved despite their widespread use for monitoring RLVR training~\cite{luo2025deepscaler,liu2025prorl,guo2025deepseek,yue2025vapo}. 

We note that although $\mathrm{Max\ Gap}$ intrinsically depends on benchmark size, this reflects a real coverage limitation rather than merely a nuisance effect: even if a large gap is expected for a small benchmark, it is still an uncovered difficulty interval, and by Theorem~\ref{thm:max-gap-detection-boundary}, localized improvements inside such an interval may be undetectable in the worst case. Importantly, we find that coverage gaps are not merely a small-benchmark artifact: benchmarks with hundreds of items, such as HumanEval and BBH-5 objects, also exhibit substantial gap metrics. This motivates gap filling as a practical and necessary way to improve local benchmark resolution.

\begin{figure*}[t]
    \centering
    \includegraphics[width=\textwidth]{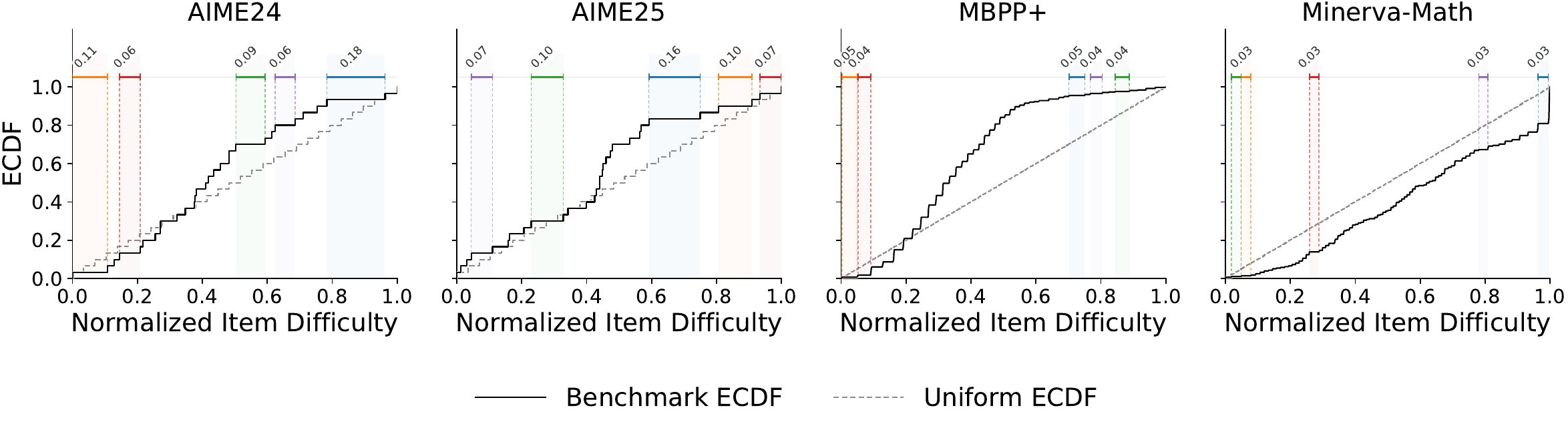}
  \vspace{-4mm}

    \caption{\textbf{Benchmark difficulty distributions.}
    Each panel shows the ECDF of normalized item difficulties for one benchmark, together with the neutral uniform reference. Colored shaded intervals mark the five largest adjacent difficulty gaps, with labels denoting the corresponding gap sizes. Additional benchmark examples are provided in Fig.~\ref{fig:appendix_result}.}
    \label{fig:main_result}
    \vspace{-4mm}
\end{figure*}

\begin{table*}[t]
\centering
\caption{\textbf{Summary of proposed metrics across benchmarks.}
$\mathrm{DDM}$ denotes the 1-Wasserstein distance between the benchmark empirical difficulty distribution and a same-size empirical uniform distribution on $[0,1]$. $\mathrm{Max\ Gap}$ is the largest gap between consecutive normalized difficulties, and $\mathrm{Top\text{-}5\ Sum}$ is the sum of the five largest such gaps. \# denotes the number of benchmark items. Lower is better for all three metrics. Highlighted entries mark comparatively large values for each metric: $\mathrm{DDM}\ge 0.10$, $\mathrm{Max\ Gap}$ $\ge 0.10$, or $\mathrm{Top\text{-}5\ Sum}$ $\ge 0.30$.}
\label{tab:main_result}
\small
\begin{minipage}[t]{0.49\linewidth}
\vspace{0pt}
\centering
\resizebox{\linewidth}{!}{
\begin{tabular}{l r r r r}
\toprule
\textbf{Benchmark} & 
\textbf{$\mathrm{DDM}$} & \textbf{$\mathrm{Max\ Gap}$} & \textbf{$\mathrm{Top\text{-}5\ Sum}$} & \textbf{\#} \\
\midrule
MATH-500 & \highddm{0.164} & 0.058 & 0.175 & 500 \\
Minerva-Math & \highddm{0.122} & 0.031 & 0.149 & 272 \\
OlympiadBench & 0.064 & 0.048 & 0.142 & 675 \\
Gaokao2023en & 0.088 & 0.055 & 0.183 & 385 \\
AMC23 & 0.089 & \highddm{0.217} & \highddm{0.539} & 40 \\
AIME24 & 0.077 & \highddm{0.179} & \highddm{0.502} & 30 \\
AIME25 & 0.079 & \highddm{0.158} & \highddm{0.492} & 30 \\
AIME26 & 0.085 & \highddm{0.119} & \highddm{0.507} & 30 \\
MMLU-Pro-Biology & \highddm{0.125} & 0.045 & 0.151 & 717 \\
MMLU-Pro-Chemistry & \highddm{0.114} & 0.072 & 0.192 & 1132 \\
MMLU-Pro-Physics & \highddm{0.121} & 0.035 & 0.157 & 1299 \\
GPQA-Diamond & 0.080 & 0.066 & 0.263 & 198 \\
ScienceQA & \highddm{0.178} & 0.045 & 0.175 & 4241 \\
\bottomrule
\end{tabular}
}
\end{minipage}
\hfill
\begin{minipage}[t]{0.49\linewidth}
\vspace{0pt}
\centering
\resizebox{\linewidth}{!}{
\begin{tabular}{l r r r r}
\toprule
\textbf{Benchmark} & 
\textbf{$\mathrm{DDM}$} & \textbf{$\mathrm{Max\ Gap}$} & \textbf{$\mathrm{Top\text{-}5\ Sum}$} & \textbf{\#} \\
\midrule
HumanEval & \highddm{0.148} & \highddm{0.165} & \highddm{0.423} & 164 \\
MBPP+ & \highddm{0.163} & 0.049 & 0.214 & 378 \\
GAIA & 0.084 & 0.051 & 0.218 & 165 \\
CORE-Bench Hard & 0.036 & \highddm{0.113} & \highddm{0.363} & 45 \\
SWE-bench-Mini & 0.053 & \highddm{0.149} & \highddm{0.469} & 50 \\
USACO & 0.080 & 0.082 & \highddm{0.354} & 307 \\
IFEval & \highddm{0.159} & 0.046 & 0.189 & 541 \\
BBH-3 objects & 0.090 & 0.075 & 0.294 & 250 \\
BBH-5 objects & 0.099 & \highddm{0.128} & \highddm{0.332} & 250 \\
BBH-7 objects & 0.077 & 0.059 & 0.279 & 250 \\
ZebraLogic & 0.080 & 0.027 & 0.127 & 1000 \\
MedQA & \highddm{0.157} & 0.042 & 0.178 & 1273 \\
\bottomrule
\end{tabular}
}
\end{minipage}
    \vspace{-4mm}

\end{table*}
\section{\textsc{BenchFill}: Difficulty-Targeted Gap Filling}

\begin{figure*}[t]
    \centering
    \includegraphics[width=\textwidth]{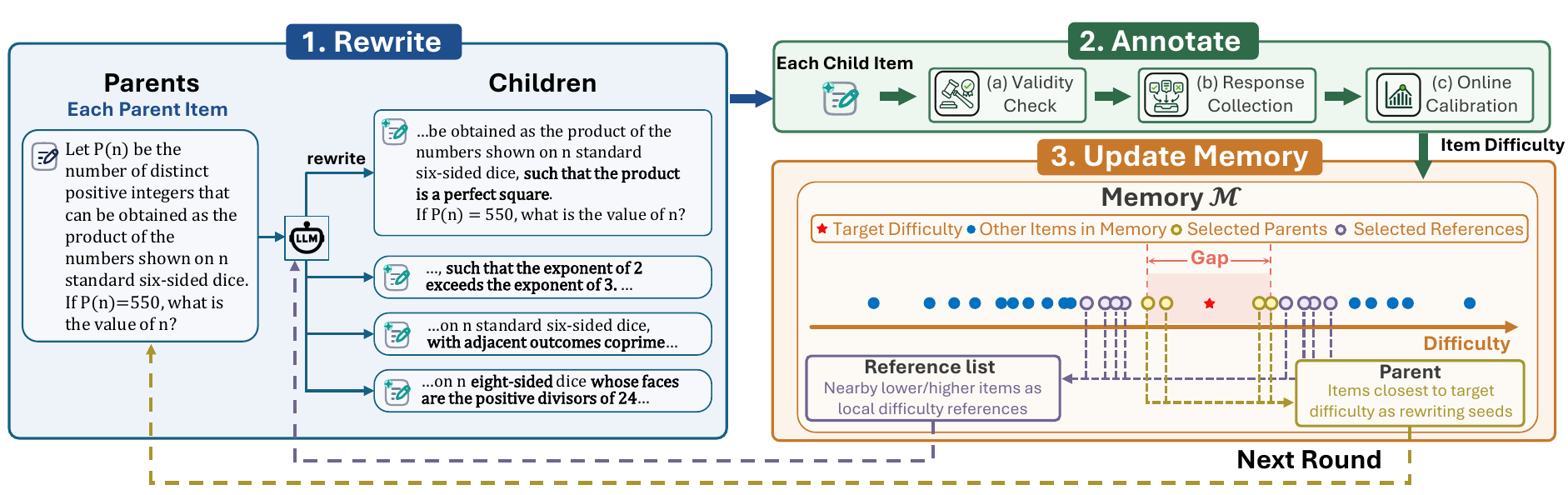}
      \vspace{-4mm}
    \caption{\textbf{Overview of \textsc{BenchFill}.}
    \textsc{BenchFill} iteratively rewrites parent items into child variants, annotates their difficulties, and updates the memory to guide the next round of gap filling.}
    \label{fig:benchfill_overview}
  \vspace{-4mm}
\end{figure*}

\subsection{Problem Setting and Method}


\paragraph{Problem setting.}
Let $B=\{x_i\}_{i=1}^n$ be the current benchmark, where each item $x_i$ has an estimated difficulty $z_i$. We sort the difficulties and consider an adjacent gap $[z_L,z_R]$ between two benchmark items. Reducing such gaps improves a benchmark's local resolution and its ability to capture localized model improvements. Thus, our goal is to obtain a valid new item $x_{\mathrm{new}}$ whose difficulty is as close as possible to the target difficulty $z^\star=(z_L+z_R)/2$, the midpoint that most reduces the gap under a single insertion. The new item is added to the benchmark rather than replacing any existing item. By default, we apply this procedure to the \textit{largest} gap of the benchmark. 

One direct approach is to ask human experts to design new items at the target difficulty, but this process can be costly~\cite{rein2024,phan2025humanity}. Motivated by recent work on LLM-assisted benchmark augmentation~\cite{cheng2026vera,li2025ride,ying2024automating,xia2024top}, we instead use LLMs to \textit{rewrite} existing benchmark items into new variants at targeted difficulty levels. This rewriting-based approach is scalable, uses existing items as local references for controlled difficulty adjustment, and helps preserve the format and style of the original benchmark. To ensure item quality, the final selected items are further validated by human experts. We next introduce \textsc{BenchFill}, a rewrite-and-annotate evolutionary search algorithm~\citep{sharma5openevolve,liu2026evox} for automated difficulty-targeted item rewriting. Alg.~\ref{alg:gap_filling} gives the pseudocode, and Fig.~\ref{fig:benchfill_overview} illustrates the procedure.

\paragraph{\textsc{BenchFill}.}
Given a target difficulty $z^\star$, our algorithm maintains a memory $\mathcal{M}$ of annotated item--difficulty pairs. The memory is initialized with existing benchmark items whose difficulties are closest to the selected gap, and is later updated during the search. Each round follows a parent--child rewriting step. A \emph{parent item} is an existing or previously rewritten item selected from $\mathcal{M}$ as the main reference for rewriting, while a \emph{child item} is a newly proposed variant derived from a parent.

At round $t$, we retrieve a small set of parent items from $\mathcal{M}$ whose annotated difficulties are closest to the target difficulty $z^\star$. For each parent item, the rewriter LLM is conditioned on the parent, the target difficulty, and a reference list of nearby item--difficulty pairs. The reference list provides local examples around $z^\star$, typically including items slightly below and above the target difficulty, e.g., items with difficulties $0.48,0.50,0.60$, and $0.62$ when $z^\star=0.55$. The rewriter then rewrites each parent into a batch of child items while preserving the format and style of the benchmark.

After rewriting, each child item is passed to the difficulty annotator module, detailed below, which returns an estimated difficulty $\hat{z}$. We keep only valid\footnote{We use GPT-5.4 as a validity judge to verify whether the answer associated with each rewritten item is correct.} child items and add their annotated item--difficulty pairs $(x,\hat{z})$ back to $\mathcal{M}$. This update makes the search adaptive: whenever a rewritten child is closer to $z^\star$ than previous items, it becomes more likely to be retrieved as a parent or included in the reference list in later rounds. The rewriter is therefore conditioned on increasingly target-relevant examples as the search progresses, rather than rewriting each batch independently. After $T$ rounds, we select the valid child item minimizing $|\hat{z}-z^\star|$ as the proposed gap-filling item. Before insertion into the benchmark, the final item is subject to careful manual review by human experts.

\paragraph{Difficulty annotation module details.}
For each rewritten item $x$, we annotate its difficulty by online calibration under the IRT model fitted on the original benchmark. Let $\hat{\theta}_j$ be the pre-estimated ability of model $j$. We evaluate each model on $x$ for $K_j$ trials and observe $c_j$ correct responses. The item difficulty is estimated by fitting only the new item parameter:
\[
\hat{b}(x)
=
\arg\max_b
\sum_{j=1}^{m}
\left[
c_j\log \sigma(\hat{\theta}_j-b)
+
\bigl(K_j-c_j\bigr)\log\bigl(1-\sigma(\hat{\theta}_j-b)\bigr)
\right],
\]
where $\sigma(\cdot)$ is the logistic function\footnote{For multiple-choice benchmarks, we additionally incorporate the same chance-level guessing parameter used in the original IRT fitting, which is omitted here for notational simplicity.}. This online calibration places each rewritten item on the same IRT scale as the original benchmark without refitting the full model.

\subsection{Experimental Setup and Results}

\paragraph{Experimental setup.}
We perform gap filling on 8 representative benchmarks from diverse domains. For each benchmark, we target the three most pronounced coverage gaps. We use MiniMax-M2.7~\cite{chen2026minimax} as the rewriter LLM, with temperature set to $1.0$. The search runs for at most 8 rounds; in each round, we retrieve 4 parent items, rewrite each parent into 4 child items, and provide a reference list of at most 16 nearby item--difficulty pairs. As a baseline, we use a budget-matched best-of-\(N\) uniform-parent rewriting variant. Additional implementation details are provided in App.~\ref{sec:app_gap}.


\paragraph{Results.}
As shown in Tab.~\ref{tab:gap_filling_results}, with only three inserted items per benchmark, \textsc{BenchFill} reduces $\mathrm{Top\text{-}3\ Sum}$ by $21.09\%$--$38.88\%$. We also include the oracle result, but only as a reference upper bound, not as an implementable gap-filling method. The achieved reductions are close to this oracle reference: on average, \textsc{BenchFill} achieves a $32.67\%$ reduction, compared with $35.06\%$ for the oracle and $17.52\%$ for the budget-matched uniform-parent baseline. On AIME24, AIME25, and AIME26, \textsc{BenchFill} reaches the oracle reduction. These reductions show that \textsc{BenchFill} improves coverage of the largest under-covered difficulty regions, which can strengthen a benchmark's ability to monitor localized training progress. We further report Target Dev. and \textsc{BenchFill} consistently achieves a Target Dev. below \(0.01\), indicating that our proposed algorithm can produce items close to the intended difficulty levels. Representative generated items are shown in Tab.~\ref{tab:gap_filling_examples}. 

\begin{table*}[!t]
\centering
\caption{\textbf{Gap-filling results.}
Top-3 Gap Sum Reduction reports the relative reduction (\%) from the original benchmark, where larger is better. Target Dev. is the mean absolute deviation between the selected item difficulty and the target difficulty over three filled gaps, where lower is better. Baseline uses a budget-matched best-of-\(N\) uniform-parent rewriting strategy. Oracle is an \textit{idealized upper bound} that sequentially inserts three items, each at the midpoint of the current largest gap, rather than an actual gap-filling method.
}
\label{tab:gap_filling_results}
\small
\resizebox{0.8\textwidth}{!}{
\begin{tabular}{lccccc}
\toprule
\textbf{Benchmark}
& \multicolumn{3}{c}{\textbf{Top-3 Gap Sum Reduction (\%)} $\uparrow$}
& \multicolumn{2}{c}{\textbf{Target Dev.} $\downarrow$} \\
\cmidrule(lr){2-4}
\cmidrule(lr){5-6}
& \textbf{Baseline} & \textbf{\textsc{BenchFill}} & \cellcolor{gray!15}\textbf{Oracle}
& \textbf{Baseline} & \textbf{\textsc{BenchFill}} \\
\midrule
AIME24 & 6.90 & 35.57 & \cellcolor{gray!15}{35.57} & 0.057 & 0.005 \\
AIME25 & 27.44 & 37.56 & \cellcolor{gray!15}{37.56} & 0.016 & 0.005 \\
AIME26 & 17.94 & 33.04 & \cellcolor{gray!15}{33.04} & 0.045 & 0.002 \\
AMC23 & 24.47 & 35.19 & \cellcolor{gray!15}{35.78} & 0.009 & 0.003 \\
Gaokao2023en & 11.45 & 38.88 & \cellcolor{gray!15}{45.47} & 0.152 & 0.004 \\
BBH-3 objects & 14.83 & 29.88 & \cellcolor{gray!15}{32.80} & 0.026 & 0.006 \\
MMLU-Pro-Biology & 22.23 & 30.13 & \cellcolor{gray!15}{36.84} & 0.046 & 0.006 \\
IFEval & 14.94 & 21.09 & \cellcolor{gray!15}{23.44} & 0.036 & 0.007 \\
\midrule
Average & 17.52 & 32.67 & \cellcolor{gray!15}{35.06} & 0.048 & 0.005 \\
\bottomrule
\end{tabular}
}
    \vspace{-5mm}
\end{table*}

\begin{figure*}[!t]
    \centering
    \includegraphics[width=\linewidth]{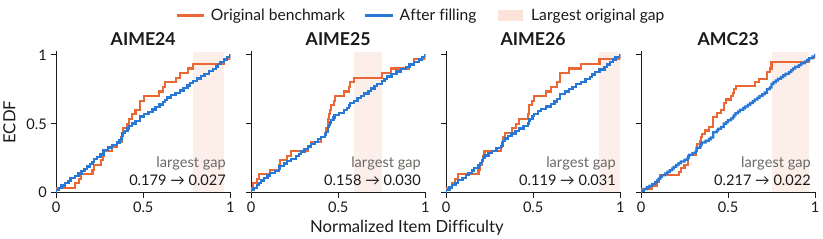}
      \vspace{-4mm}

    \caption{\textbf{Difficulty distributions of four benchmarks before and after gap filling.} Orange and blue curves show the empirical difficulty distributions before and after applying \textsc{BenchFill}, respectively. Shaded regions indicate the largest original coverage gaps. The annotated values show the maximum gaps before and after filling.}
    \label{fig:curated_benchmarks}
        \vspace{-5mm}
\end{figure*}

\subsection{Curated Gap-Filled Benchmarks}
\label{sec:curated_benchmarks}

To provide a community resource for more reliable evaluation of future models and training methods, we further apply \textsc{BenchFill} to construct gap-filled variants of four widely used math benchmarks with substantial coverage gaps: AMC23, AIME24, AIME25, and AIME26. We expand these benchmarks from 40, 30, 30, and 30 items to 80, 50, 52, and 53 items, respectively. All generated items are manually reviewed, and low-quality or incorrect items are discarded. We release the resulting variants as \textbf{AMC23-BenchFill}, \textbf{AIME24-BenchFill}, \textbf{AIME25-BenchFill}, and \textbf{AIME26-BenchFill}\footnote{\url{https://huggingface.co/datasets/Yifan99/BenchFill}}. As shown in Fig.~\ref{fig:curated_benchmarks}, these variants substantially reduce the largest original coverage gaps and bring the empirical difficulty distributions closer to a uniform distribution, providing more comprehensive coverage across the difficulty spectrum.

\section{Related Work}

\paragraph{Science of benchmarking.}
A recent line of work has shifted attention from improving leaderboard scores to examining whether benchmark scores themselves provide reliable scientific evidence. Beyond data contamination, where benchmark test items appear in model training data and potentially inflate evaluation scores~\cite{oren2024proving,sun2025emperor,zhu2024dynamic}, prior studies show that benchmark results can be sensitive to decoding, prompting, and implementation details~\cite{hochlehnert2025sober}, and brittle to surface wording or dataset flaws~\cite{mousavi2026garbage}. Other work studies evaluation reliability through the lens of signal and noise~\cite{heineman2025signal}, or proposes broader benchmark-quality criteria and shows that widely used benchmarks vary substantially in quality~\cite{qian2026benchmark}. Our work complements these efforts by formulating a benchmark as an empirical measure over difficulty, enabling us to rigorously formalize measure mismatch and coverage gaps.

\paragraph{IRT and its applications.} IRT originated in psychometrics as a framework for educational and psychological measurement~\cite{embretson2013item,lord2012applications}. In NLP, \cite{lalor2016building} introduced IRT to system evaluation, and \cite{rodriguez2021evaluation} extended this perspective to leaderboard analysis. IRT has since become a widely used foundation for difficulty-aware evaluation in LLM benchmarking. Easy2Hard-Bench~\cite{ding2024easy2hard} uses IRT to assign standardized continuous difficulty labels to benchmark items; it also validates these labels against human pairwise judgments and reports stronger alignment than GPT-4-based difficulty annotators. \cite{zhou2026lost} uses an IRT-based framework to analyze the measurement quality of existing LLM benchmarks. 

\section{Conclusion}

In this work, we view each benchmark as an empirical measure over difficulty. This perspective enables us to formalize measure mismatch and coverage gaps, leading to theoretical results and practical diagnostics. 
We believe this empirical-measure view has the potential to open new ways to rigorously study and rethink benchmark evaluation.

\section*{Acknowledgments}
We thank Modal for supporting this research through the Modal for Academics compute grant. We are also grateful to Sirui Chen and Junsheng Huang for their insightful discussions and valuable feedback, which helped improve this work.

\bibliographystyle{plain}
\small\bibliography{neurips_2026}

@article{ding2024easy2hard,
  title={Easy2Hard-bench: Standardized difficulty labels for profiling LLM performance and generalization},
  author={Ding, Mucong and Deng, Chenghao and Choo, Jocelyn and Wu, Zichu and Agrawal, Aakriti and Schwarzschild, Avi and Zhou, Tianyi and Goldstein, Tom and Langford, John and Anandkumar, Anima and others},
  journal={Advances in Neural Information Processing Systems},
  volume={37},
  pages={44323--44365},
  year={2024}
}

@inproceedings{rodriguez2021evaluation,
  title={Evaluation examples are not equally informative: How should that change NLP leaderboards?},
  author={Rodriguez, Pedro and Barrow, Joe and Hoyle, Alexander Miserlis and Lalor, John P and Jia, Robin and Boyd-Graber, Jordan},
  booktitle={Proceedings of the 59th Annual Meeting of the Association for Computational Linguistics and the 11th International Joint Conference on Natural Language Processing (Volume 1: Long Papers)},
  pages={4486--4503},
  year={2021}
}

@book{rasch1993probabilistic,
  title={Probabilistic models for some intelligence and attainment tests.},
  author={Rasch, Georg},
  year={1993},
  publisher={ERIC}
}

@article{lord1952theory,
  title={A theory of test scores.},
  author={Lord, Frederic},
  journal={Psychometric monographs},
  year={1952}
}

@article{polo2024tinybenchmarks,
  title={tinyBenchmarks: evaluating LLMs with fewer examples},
  author={Polo, Felipe Maia and Weber, Lucas and Choshen, Leshem and Sun, Yuekai and Xu, Gongjun and Yurochkin, Mikhail},
  journal={arXiv preprint arXiv:2402.14992},
  year={2024}
}

@article{hofmann2025fluid,
  title={Fluid language model benchmarking},
  author={Hofmann, Valentin and Heineman, David and Magnusson, Ian and Lo, Kyle and Dodge, Jesse and Sap, Maarten and Koh, Pang Wei and Wang, Chun and Hajishirzi, Hannaneh and Smith, Noah A},
  journal={arXiv preprint arXiv:2509.11106},
  year={2025}
}

@article{kipnis2024metabench,
  title={metabench--A Sparse Benchmark of Reasoning and Knowledge in Large Language Models},
  author={Kipnis, Alex and Voudouris, Konstantinos and Buschoff, Luca M Schulze and Schulz, Eric},
  journal={arXiv preprint arXiv:2407.12844},
  year={2024}
}

@inproceedings{zhou2026lost,
  title={Lost in benchmarks? rethinking large language model benchmarking with item response theory},
  author={Zhou, Hongli and Huang, Hui and Zhao, Ziqing and Han, Lvyuan and Wang, Huicheng and Chen, Kehai and Yang, Muyun and Bao, Wei and Dong, Jian and Xu, Bing and others},
  booktitle={Proceedings of the AAAI Conference on Artificial Intelligence},
  volume={40},
  number={41},
  pages={35085--35093},
  year={2026}
}

@article{pymc2023,
  title = {{PyMC}: A Modern and Comprehensive Probabilistic Programming Framework in {P}ython},
  author = {Oriol Abril-Pla and Virgile Andreani and Colin Carroll and Larry Dong and Christopher J. Fonnesbeck and Maxim Kochurov and Ravin Kumar and Junpeng Lao and Christian C. Luhmann and Osvaldo A. Martin and Michael Osthege and Ricardo Vieira and Thomas Wiecki and Robert Zinkov },
  journal = {{PeerJ} Computer Science},
  volume = {9},
  number = {e1516},
  doi = {10.7717/peerj-cs.1516},
  year = {2023}
}

@book{baker2001basics,
  title={The basics of item response theory},
  author={Baker, Frank B},
  year={2001},
  publisher={ERIC}
}

@article{sijtsma1998nonparametric,
  title={Nonparametric polytomous IRT models for invariant item ordering, with results for parametric models},
  author={Sijtsma, Klaas and Hemker, Bas T},
  journal={Psychometrika},
  volume={63},
  number={2},
  pages={183--200},
  year={1998},
  publisher={Springer-Verlag}
}

@article{sijtsma1996survey,
  title={A survey of theory and methods of invariant item ordering},
  author={Sijtsma, Klaas and Junker, Brian W},
  journal={British Journal of Mathematical and Statistical Psychology},
  volume={49},
  number={1},
  pages={79--105},
  year={1996},
  publisher={Wiley Online Library}
}

@inproceedings{lalor2016building,
  title={Building an evaluation scale using item response theory},
  author={Lalor, John P and Wu, Hao and Yu, Hong},
  booktitle={Proceedings of the 2016 Conference on Empirical Methods in Natural Language Processing},
  pages={648--657},
  year={2016}
}

@inproceedings{wang2018glue,
  title={GLUE: A multi-task benchmark and analysis platform for natural language understanding},
  author={Wang, Alex and Singh, Amanpreet and Michael, Julian and Hill, Felix and Levy, Omer and Bowman, Samuel},
  booktitle={Proceedings of the 2018 EMNLP workshop BlackboxNLP: Analyzing and interpreting neural networks for NLP},
  pages={353--355},
  year={2018}
}

@article{liang2022holistic,
  title={Holistic evaluation of language models},
  author={Liang, Percy and Bommasani, Rishi and Lee, Tony and Tsipras, Dimitris and Soylu, Dilara and Yasunaga, Michihiro and Zhang, Yian and Narayanan, Deepak and Wu, Yuhuai and Kumar, Ananya and others},
  journal={arXiv preprint arXiv:2211.09110},
  year={2022}
}

@inproceedings{alzahrani2024benchmarks,
  title={When benchmarks are targets: Revealing the sensitivity of large language model leaderboards},
  author={Alzahrani, Norah and Alyahya, Hisham and Alnumay, Yazeed and Alrashed, Sultan and Alsubaie, Shaykhah and Almushayqih, Yousef and Mirza, Faisal and Alotaibi, Nouf and Al-Twairesh, Nora and Alowisheq, Areeb and others},
  booktitle={Proceedings of the 62nd Annual Meeting of the Association for Computational Linguistics (Volume 1: Long Papers)},
  pages={13787--13805},
  year={2024}
}

@article{myrzakhan2024open,
  title={Open-llm-leaderboard: From multi-choice to open-style questions for llms evaluation, benchmark, and arena},
  author={Myrzakhan, Aidar and Bsharat, Sondos Mahmoud and Shen, Zhiqiang},
  journal={arXiv preprint arXiv:2406.07545},
  year={2024}
}

@article{liu2025prorl,
  title={Prorl: Prolonged reinforcement learning expands reasoning boundaries in large language models},
  author={Liu, Mingjie and Diao, Shizhe and Lu, Ximing and Hu, Jian and Dong, Xin and Choi, Yejin and Kautz, Jan and Dong, Yi},
  journal={arXiv preprint arXiv:2505.24864},
  year={2025}
}

@article{hu2025open,
  title={Open-reasoner-zero: An open source approach to scaling up reinforcement learning on the base model},
  author={Hu, Jingcheng and Zhang, Yinmin and Han, Qi and Jiang, Daxin and Zhang, Xiangyu and Shum, Heung-Yeung},
  journal={arXiv preprint arXiv:2503.24290},
  year={2025}
}

@article{he2025skywork,
  title={Skywork open reasoner 1 technical report},
  author={He, Jujie and Liu, Jiacai and Liu, Chris Yuhao and Yan, Rui and Wang, Chaojie and Cheng, Peng and Zhang, Xiaoyu and Zhang, Fuxiang and Xu, Jiacheng and Shen, Wei and others},
  journal={arXiv preprint arXiv:2505.22312},
  year={2025}
}

@article{yu2025dapo,
  title={Dapo: An open-source llm reinforcement learning system at scale},
  author={Yu, Qiying and Zhang, Zheng and Zhu, Ruofei and Yuan, Yufeng and Zuo, Xiaochen and Yue, Yu and Dai, Weinan and Fan, Tiantian and Liu, Gaohong and Liu, Lingjun and others},
  journal={arXiv preprint arXiv:2503.14476},
  year={2025}
}

@article{bercovich2025llama,
  title={Llama-nemotron: Efficient reasoning models},
  author={Bercovich, Akhiad and Levy, Itay and Golan, Izik and Dabbah, Mohammad and El-Yaniv, Ran and Puny, Omri and Galil, Ido and Moshe, Zach and Ronen, Tomer and Nabwani, Najeeb and others},
  journal={arXiv preprint arXiv:2505.00949},
  year={2025}
}

@misc{luo2025deepscaler,
  title={DeepScaleR: Surpassing O1-Preview with a 1.5B Model by Scaling RL},
  author={Michael Luo and Sijun Tan and Justin Wong and Xiaoxiang Shi and William Y. Tang and Manan Roongta and Colin Cai and Jeffrey Luo and Li Erran Li and Raluca Ada Popa and Ion Stoica},
  howpublished={\url{https://pretty-radio-b75.notion.site/DeepScaleR-Surpassing-O1-Preview-with-a-1-5B-Model-by-Scaling-RL-19681902c1468005bed8ca303013a4e2}},
  note={Notion Blog},
  year={2025}
}

@inproceedings{siska2024examining,
  title={Examining the robustness of LLM evaluation to the distributional assumptions of benchmarks},
  author={Siska, Charlotte and Marazopoulou, Katerina and Ailem, Melissa and Bono, James},
  booktitle={Proceedings of the 62nd Annual Meeting of the Association for Computational Linguistics (Volume 1: Long Papers)},
  pages={10406--10421},
  year={2024}
}

@article{zhang2026rankllm,
  title={RankLLM: Weighted Ranking of LLMs by Quantifying Question Difficulty},
  author={Zhang, Ziqian and Hu, Xingjian and Huang, Yue and Zhang, Kai and Chen, Ruoxi and Liu, Yixin and Wen, Qingsong and Xu, Kaidi and Zhang, Xiangliang and Gong, Neil Zhenqiang and others},
  journal={arXiv preprint arXiv:2602.12424},
  year={2026}
}

@book{villani2009optimal,
  title={Optimal transport: old and new},
  author={Villani, C{\'e}dric and others},
  volume={338},
  year={2009},
  publisher={Springer}
}

@article{chen2025item,
  title={Item response theory—A statistical framework for educational and psychological measurement},
  author={Chen, Yunxiao and Li, Xiaoou and Liu, Jingchen and Ying, Zhiliang},
  journal={Statistical Science},
  volume={40},
  number={2},
  pages={167--194},
  year={2025},
  publisher={Institute of Mathematical Statistics}
}

@article{ballon2025estimating,
  title={Estimating problem difficulty without ground truth using Large Language Model comparisons},
  author={Ballon, Marthe and Algaba, Andres and Verbeken, Brecht and Ginis, Vincent},
  journal={arXiv preprint arXiv:2512.14220},
  year={2025}
}

@inproceedings{zhu2025llm,
  title={The llm already knows: Estimating llm-perceived question difficulty via hidden representations},
  author={Zhu, Yubo and Liu, Dongrui and Lin, Zecheng and Tong, Wei and Zhong, Sheng and Shao, Jing},
  booktitle={Proceedings of the 2025 Conference on Empirical Methods in Natural Language Processing},
  pages={1160--1176},
  year={2025}
}

@inproceedings{rein2024gpqa,
  title={Gpqa: A graduate-level google-proof q\&a benchmark},
  author={Rein, David and Hou, Betty Li and Stickland, Asa Cooper and Petty, Jackson and Pang, Richard Yuanzhe and Dirani, Julien and Michael, Julian and Bowman, Samuel R},
  booktitle={First conference on language modeling},
  year={2024}
}

@inproceedings{lightman2023let,
  title={Let's verify step by step},
  author={Lightman, Hunter and Kosaraju, Vineet and Burda, Yuri and Edwards, Harrison and Baker, Bowen and Lee, Teddy and Leike, Jan and Schulman, John and Sutskever, Ilya and Cobbe, Karl},
  booktitle={International Conference on Learning Representations},
  volume={2024},
  pages={39578--39601},
  year={2024}
}

@article{bean2024lingoly,
  title={LINGOLY: A benchmark of olympiad-level linguistic reasoning puzzles in low resource and extinct languages},
  author={Bean, Andrew and Hellsten, Simi and Mayne, Harry and Magomere, Jabez and Chi, Ryan and Hale, Scott A and Kirk, Hannah R and others},
  journal={Advances in Neural Information Processing Systems},
  volume={37},
  pages={26224--26237},
  year={2024}
}

@book{embretson2013item,
  title={Item response theory for psychologists},
  author={Embretson, Susan E and Reise, Steven P},
  year={2013},
  publisher={Psychology Press}
}

@book{lord2012applications,
  title={Applications of item response theory to practical testing problems},
  author={Lord, Frederic M},
  year={2012},
  publisher={Routledge}
}

@article{hochlehnert2025sober,
  title={A sober look at progress in language model reasoning: Pitfalls and paths to reproducibility},
  author={Hochlehnert, Andreas and Bhatnagar, Hardik and Udandarao, Vishaal and Albanie, Samuel and Prabhu, Ameya and Bethge, Matthias},
  journal={arXiv preprint arXiv:2504.07086},
  year={2025}
}

@inproceedings{mousavi2026garbage,
  title={Garbage in, reasoning out? why benchmark scores are unreliable and what to do about it},
  author={Mousavi, Seyed Mahed and Cecchinato, Edoardo and Hornikova, Lucia and Riccardi, Giuseppe},
  booktitle={Findings of the Association for Computational Linguistics: EACL 2026},
  pages={1747--1759},
  year={2026}
}

@article{qian2026benchmark,
  title={Benchmark\^{} 2: Systematic Evaluation of LLM Benchmarks},
  author={Qian, Qi and Huang, Chengsong and Xu, Jingwen and Lv, Changze and Wu, Muling and Liu, Wenhao and Wang, Xiaohua and Wang, Zhenghua and Huang, Zisu and Tian, Muzhao and others},
  journal={arXiv preprint arXiv:2601.03986},
  year={2026}
}

@article{heineman2025signal,
  title={Signal and noise: A framework for reducing uncertainty in language model evaluation},
  author={Heineman, David and Hofmann, Valentin and Magnusson, Ian and Gu, Yuling and Smith, Noah A and Hajishirzi, Hannaneh and Lo, Kyle and Dodge, Jesse},
  journal={arXiv preprint arXiv:2508.13144},
  year={2025}
}

@inproceedings{mialon2023gaia,
  title={Gaia: a benchmark for general ai assistants},
  author={Mialon, Gr{\'e}goire and Fourrier, Cl{\'e}mentine and Wolf, Thomas and LeCun, Yann and Scialom, Thomas},
  booktitle={International Conference on Learning Representations},
  volume={2024},
  pages={9025--9049},
  year={2024}
}

@misc{aime24,
  author       = {{Mathematical Association of America}},
  title        = {{2024 American Invitational Mathematics Examination (AIME)}},
  year         = {2024},
  howpublished = {\url{https://maa.org/maa-invitational-competitions/}},
  note         = {Accessed 2026-04-25}
}

@misc{aime25,
  author       = {{Mathematical Association of America}},
  title        = {{2025 American Invitational Mathematics Examination (AIME)}},
  year         = {2025},
  howpublished = {\url{https://maa.org/maa-invitational-competitions/}},
  note         = {Accessed 2026-04-25}
}

@misc{aime26,
  author       = {{Mathematical Association of America}},
  title        = {{2026 American Invitational Mathematics Examination (AIME)}},
  year         = {2026},
  howpublished = {\url{https://maa.org/maa-invitational-competitions/}},
  note         = {Accessed 2026-04-25}
}

@inproceedings{he2024olympiadbench,
  title={Olympiadbench: A challenging benchmark for promoting agi with olympiad-level bilingual multimodal scientific problems},
  author={He, Chaoqun and Luo, Renjie and Bai, Yuzhuo and Hu, Shengding and Thai, Zhen and Shen, Junhao and Hu, Jinyi and Han, Xu and Huang, Yujie and Zhang, Yuxiang and others},
  booktitle={Proceedings of the 62nd Annual Meeting of the Association for Computational Linguistics (Volume 1: Long Papers)},
  pages={3828--3850},
  year={2024}
}

@article{lewkowycz2022solving,
  title={Solving quantitative reasoning problems with language models},
  author={Lewkowycz, Aitor and Andreassen, Anders and Dohan, David and Dyer, Ethan and Michalewski, Henryk and Ramasesh, Vinay and Slone, Ambrose and Anil, Cem and Schlag, Imanol and Gutman-Solo, Theo and others},
  journal={Advances in neural information processing systems},
  volume={35},
  pages={3843--3857},
  year={2022}
}

@article{hendrycks2021measuring,
  title={Measuring mathematical problem solving with the math dataset},
  author={Hendrycks, Dan and Burns, Collin and Kadavath, Saurav and Arora, Akul and Basart, Steven and Tang, Eric and Song, Dawn and Steinhardt, Jacob},
  journal={arXiv preprint arXiv:2103.03874},
  year={2021}
}

@article{wang2024mmlu,
  title={Mmlu-pro: A more robust and challenging multi-task language understanding benchmark},
  author={Wang, Yubo and Ma, Xueguang and Zhang, Ge and Ni, Yuansheng and Chandra, Abhranil and Guo, Shiguang and Ren, Weiming and Arulraj, Aaran and He, Xuan and Jiang, Ziyan and others},
  journal={Advances in Neural Information Processing Systems},
  volume={37},
  pages={95266--95290},
  year={2024}
}

@article{zhou2023instruction,
  title={Instruction-following evaluation for large language models},
  author={Zhou, Jeffrey and Lu, Tianjian and Mishra, Swaroop and Brahma, Siddhartha and Basu, Sujoy and Luan, Yi and Zhou, Denny and Hou, Le},
  journal={arXiv preprint arXiv:2311.07911},
  year={2023}
}

@inproceedings{lu2022learn,
    title={Learn to Explain: Multimodal Reasoning via Thought Chains for Science Question Answering},
    author={Lu, Pan and Mishra, Swaroop and Xia, Tony and Qiu, Liang and Chang, Kai-Wei and Zhu, Song-Chun and Tafjord, Oyvind and Clark, Peter and Ashwin Kalyan},
    booktitle={The 36th Conference on Neural Information Processing Systems (NeurIPS)},
    year={2022}
}

@article{sharma5openevolve,
  title={Openevolve: an open-source evolutionary coding agent, 2025},
  author={Sharma, Asankhaya},
  journal={URL https://github. com/codelion/openevolve},
}

@article{liu2026evox,
  title={Evox: Meta-evolution for automated discovery},
  author={Liu, Shu and Agarwal, Shubham and Maheswaran, Monishwaran and Cemri, Mert and Li, Zhifei and Mang, Qiuyang and Naren, Ashwin and Boneh, Ethan and Cheng, Audrey and Pan, Melissa Z and others},
  journal={arXiv preprint arXiv:2602.23413},
  year={2026}
}

@article{zhang2023evaluating,
  title={Evaluating the performance of large language models on gaokao benchmark},
  author={Zhang, Xiaotian and Li, Chunyang and Zong, Yi and Ying, Zhengyu and He, Liang and Qiu, Xipeng},
  journal={arXiv preprint arXiv:2305.12474},
  year={2023}
}

@article{chen2021codex,
  title={Evaluating Large Language Models Trained on Code},
  author={Mark Chen and Jerry Tworek and Heewoo Jun and Qiming Yuan and Henrique Ponde de Oliveira Pinto and Jared Kaplan and Harri Edwards and Yuri Burda and Nicholas Joseph and Greg Brockman and Alex Ray and Raul Puri and Gretchen Krueger and Michael Petrov and Heidy Khlaaf and Girish Sastry and Pamela Mishkin and Brooke Chan and Scott Gray and Nick Ryder and Mikhail Pavlov and Alethea Power and Lukasz Kaiser and Mohammad Bavarian and Clemens Winter and Philippe Tillet and Felipe Petroski Such and Dave Cummings and Matthias Plappert and Fotios Chantzis and Elizabeth Barnes and Ariel Herbert-Voss and William Hebgen Guss and Alex Nichol and Alex Paino and Nikolas Tezak and Jie Tang and Igor Babuschkin and Suchir Balaji and Shantanu Jain and William Saunders and Christopher Hesse and Andrew N. Carr and Jan Leike and Josh Achiam and Vedant Misra and Evan Morikawa and Alec Radford and Matthew Knight and Miles Brundage and Mira Murati and Katie Mayer and Peter Welinder and Bob McGrew and Dario Amodei and Sam McCandlish and Ilya Sutskever and Wojciech Zaremba},
  year={2021},
  eprint={2107.03374},
  archivePrefix={arXiv},
  primaryClass={cs.LG}
}

@misc{amc23,
  author       = {{Mathematical Association of America}},
  title        = {{2023 American Mathematics Competitions 12 (AMC 12)}},
  year         = {2023},
  howpublished = {\url{https://maa.org/student-programs/amc/}},
  note         = {Accessed 2026-04-25}
}

@article{siegel2024core,
  title={Core-bench: Fostering the credibility of published research through a computational reproducibility agent benchmark},
  author={Siegel, Zachary S and Kapoor, Sayash and Nadgir, Nitya and Stroebl, Benedikt and Narayanan, Arvind},
  journal={arXiv preprint arXiv:2409.11363},
  year={2024}
}

@article{jimenez2023swe,
  title={Swe-bench: Can language models resolve real-world github issues?},
  author={Jimenez, Carlos E and Yang, John and Wettig, Alexander and Yao, Shunyu and Pei, Kexin and Press, Ofir and Narasimhan, Karthik},
  journal={arXiv preprint arXiv:2310.06770},
  year={2023}
}

@article{shi2024can,
  title={Can language models solve olympiad programming?},
  author={Shi, Quan and Tang, Michael and Narasimhan, Karthik and Yao, Shunyu},
  journal={arXiv preprint arXiv:2404.10952},
  year={2024}
}

@article{suzgun2022challenging,
  title={Challenging BIG-Bench Tasks and Whether Chain-of-Thought Can Solve Them},
  author={Suzgun, Mirac and Scales, Nathan and Sch{\"a}rli, Nathanael and Gehrmann, Sebastian and Tay, Yi and Chung, Hyung Won and Chowdhery, Aakanksha and Le, Quoc V and Chi, Ed H and Zhou, Denny and Wei, Jason},
  journal={arXiv preprint arXiv:2210.09261},
  year={2022}
}

@article{zebralogic2025,
    title={ZebraLogic: On the Scaling Limits of LLMs for Logical Reasoning},
    author={Bill Yuchen Lin and Ronan Le Bras and Kyle Richardson and Ashish Sabharwal and Radha Poovendran and Peter Clark and Yejin Choi},
    year={2025},
    url={https://arxiv.org/abs/2502.01100},
}

@article{jin2021disease,
  title={What disease does this patient have? a large-scale open domain question answering dataset from medical exams},
  author={Jin, Di and Pan, Eileen and Oufattole, Nassim and Weng, Wei-Hung and Fang, Hanyi and Szolovits, Peter},
  journal={Applied Sciences},
  volume={11},
  number={14},
  pages={6421},
  year={2021},
  publisher={MDPI}
}

@article{phan2025humanity,
  title={Humanity's last exam},
  author={Phan, Long and Gatti, Alice and Han, Ziwen and Li, Nathaniel and Hu, Josephina and Zhang, Hugh and Zhang, Chen Bo Calvin and Shaaban, Mohamed and Ling, John and Shi, Sean and others},
  journal={arXiv preprint arXiv:2501.14249},
  year={2025}
}

@misc{rein2024,
  author       = {David Rein},
  title        = {Can Good Benchmarks Contain Mistakes?},
  year         = {2024},
  url          = {https://wp.nyu.edu/arg/can-good-benchmarks-contain-mistakes/},
}

@article{cheng2026vera,
  title={VeRA: Verified Reasoning Data Augmentation at Scale},
  author={Cheng, Zerui and Liu, Jiashuo and Wu, Chunjie and Yao, Jianzhu and Viswanath, Pramod and Zhang, Ge and Huang, Wenhao},
  journal={arXiv preprint arXiv:2602.13217},
  year={2026}
}

@article{xia2024top,
  title={Top leaderboard ranking= top coding proficiency, always? evoeval: Evolving coding benchmarks via llm},
  author={Xia, Chunqiu Steven and Deng, Yinlin and Zhang, Lingming},
  journal={arXiv preprint arXiv:2403.19114},
  year={2024}
}

@article{li2025ride,
  title={RIDE: Difficulty Evolving Perturbation with Item Response Theory for Mathematical Reasoning},
  author={Li, Xinyuan and Xu, Murong and Tao, Wenbiao and Zhu, Hanlun and Zhao, Yike and Zhang, Jipeng and Lan, Yunshi},
  journal={arXiv preprint arXiv:2511.04120},
  year={2025}
}

@inproceedings{sun2025emperor,
  title={The Emperor's New Clothes in Benchmarking? A Rigorous Examination of Mitigation Strategies for LLM Benchmark Data Contamination},
  author={Sun, Yifan and Wang, Han and Li, Dongbai and Wang, Gang and Zhang, Huan},
  booktitle={International Conference on Machine Learning},
  year={2025}
}

@article{ying2024automating,
  title={Automating dataset updates towards reliable and timely evaluation of large language models},
  author={Ying, Jiahao and Cao, Yixin and Bai, Yushi and Sun, Qianru and Wang, Bo and Tang, Wei and Ding, Zhaojun and Yang, Yizhe and Huang, Xuanjing and Yan, Shuicheng},
  journal={Advances in Neural Information Processing Systems},
  volume={37},
  pages={17106--17132},
  year={2024}
}

@article{guo2025deepseek,
  title={Deepseek-r1: Incentivizing reasoning capability in llms via reinforcement learning},
  author={Guo, Daya and Yang, Dejian and Zhang, Haowei and Song, Junxiao and Wang, Peiyi and Zhu, Qihao and Xu, Runxin and Zhang, Ruoyu and Ma, Shirong and Bi, Xiao and others},
  journal={arXiv preprint arXiv:2501.12948},
  year={2025}
}

@article{yue2025vapo,
  title={Vapo: Efficient and reliable reinforcement learning for advanced reasoning tasks},
  author={Yue, Yu and Yuan, Yufeng and Yu, Qiying and Zuo, Xiaochen and Zhu, Ruofei and Xu, Wenyuan and Chen, Jiaze and Wang, Chengyi and Fan, TianTian and Du, Zhengyin and others},
  journal={arXiv preprint arXiv:2504.05118},
  year={2025}
}

@article{liu2023your,
  title={Is your code generated by chatgpt really correct? rigorous evaluation of large language models for code generation},
  author={Liu, Jiawei and Xia, Chunqiu Steven and Wang, Yuyao and Zhang, Lingming},
  journal={Advances in neural information processing systems},
  volume={36},
  pages={21558--21572},
  year={2023}
}

@inproceedings{hal,
title =     {Holistic Agent Leaderboard: The Missing Infrastructure for AI Agent Evaluation},
author =    {Sayash Kapoor and Benedikt Stroebl and Peter Kirgis and Nitya Nadgir and Zachary S Siegel and Boyi Wei and Tianci Xue and Ziru Chen and Felix Chen and Saiteja Utpala and Franck Ndzomga and Dheeraj Oruganty and Sophie Luskin and Kangheng Liu and Botao Yu and Amit Arora and Dongyoon Hahm and Harsh Trivedi and Huan Sun and Juyong Lee and Tengjun Jin and Yifan Mai and Yifei Zhou and Yuxuan Zhu and Rishi Bommasani and Daniel Kang and Dawn Song and Peter Henderson and Yu Su and Percy Liang and Arvind Narayanan},
booktitle = {The Fourteenth International Conference on Learning Representations (ICLR)},
url =       {https://arxiv.org/pdf/2510.11977},
year =      {2026}}

@book{david2004order,
  title={Order statistics},
  author={David, Herbert A and Nagaraja, Haikady N},
  year={2004},
  publisher={John Wiley \& Sons}
}

@article{chen2026minimax,
  title={The minimax-m2 series: Mini activations unleashing max real-world intelligence},
  author={Chen, Aili and Li, Aonian and Zhou, Baichuan and Gong, Bangwei and Jiang, Binyang and Dan, Boji and Yu, Changqing and Wang, Chao and Ma, Cheng and Zhong, Cheng and others},
  journal={arXiv preprint arXiv:2605.26494},
  year={2026}
}

@article{zhu2024dynamic,
  title={Dynamic evaluation of large language models by meta probing agents},
  author={Zhu, Kaijie and Wang, Jindong and Zhao, Qinlin and Xu, Ruochen and Xie, Xing},
  journal={arXiv preprint arXiv:2402.14865},
  year={2024}
}

@inproceedings{oren2024proving,
  title={Proving test set contamination in black-box language models},
  author={Oren, Yonatan and Meister, Nicole and Chatterji, Niladri and Ladhak, Faisal and Hashimoto, Tatsunori},
  booktitle={International Conference on Learning Representations},
  volume={2024},
  pages={16354--16372},
  year={2024}
}

@inproceedings{wen2025light,
  title={Light-r1: Curriculum sft, dpo and rl for long cot from scratch and beyond},
  author={Wen, Liang and Cai, Yunke and Xiao, Fenrui and He, Xin and An, Qi and Duan, Zhenyu and Du, Yimin and Liu, Junchen and Tang, Lifu and Lv, Xiaowei and others},
  booktitle={Proceedings of the 63rd Annual Meeting of the Association for Computational Linguistics (Volume 6: Industry Track)},
  pages={318--327},
  year={2025}
}

@article{sun2026improving,
  title={Improving data efficiency for llm reinforcement fine-tuning through difficulty-targeted online data selection and rollout replay},
  author={Sun, Yifan and Shen, Jingyan and Wang, Yibin and Chen, Tianyu and Wang, Zhendong and Zhou, Mingyuan and Zhang, Huan},
  journal={Advances in Neural Information Processing Systems},
  volume={38},
  pages={173449--173476},
  year={2026}
}

@article{zhang2025reasoning,
  title={When reasoning meets its laws},
  author={Zhang, Junyu and Sun, Yifan and Leng, Tianang and Shen, Jingyan and Ziyin, Liu and Liang, Paul Pu and Zhang, Huan},
  journal={arXiv preprint arXiv:2512.17901},
  year={2025}
}


\appendix
\newpage
\section{Discussion}

\subsection{Broader Impacts}
\label{sec:app_impact}

This work aims to improve the reliability and transparency of benchmark evaluation. By diagnosing distribution mismatch and coverage gaps, our methods can help researchers better understand when benchmark scores provide faithful evidence for model comparison and monitoring training progress. More reliable evaluation may support more transparent model comparison and more informed benchmark design. Since this work focuses on benchmark analysis rather than model deployment or data collection, we do not foresee direct negative societal impacts.

\subsection{Limitations and Future Work}
\label{sec:app_limitation}

This work develops a measure-theoretic view of benchmark evaluation and uses it to characterize two central issues: measure mismatch and coverage gaps. Our goal is not to exhaust all consequences of this viewpoint, but to show that benchmark evaluation can be studied through the geometry of the empirical difficulty measure induced by the benchmark. We believe this perspective opens a broader direction for studying benchmark evaluation, where additional benchmark failures and design objectives can be formulated as properties of the induced difficulty measure.

Moreover, our current framework is primarily post-hoc: given an existing benchmark, we estimate its difficulty distribution, diagnose measure mismatch and coverage gaps, and reduce pronounced gaps through targeted item insertion. A natural next step is to move this perspective into the benchmark construction stage itself. Rather than auditing mismatch and coverage only after a benchmark has been built, benchmark developers can actively design the induced difficulty structure from the beginning, aiming to align the benchmark measure with the intended evaluation objective while maintaining high local resolution along the difficulty axis. We view this as a promising direction for difficulty-aware benchmark design.

\subsection{Extended Related Work}

\paragraph{Difficulty estimation.} Prior work often derives difficulty labels from human judgment or benchmark metadata. For example, MATH~\cite{hendrycks2021measuring} annotates problems with five difficulty levels, GPQA~\cite{rein2024gpqa} collects post-hoc expert difficulty ratings, and LingOly~\cite{bean2024lingoly} includes five levels of human difficulty. Such labels are costly to obtain at scale, and they primarily reflect human-perceived difficulty rather than the difficulty exhibited by current LLMs on the questions. Recent work has explored alternative estimation strategies. \citep{zhu2025llm} estimate LLM-perceived question difficulty from hidden representations. \cite{ballon2025estimating} ask an LLM to make pairwise difficulty judgments and aggregate them with Bradley–Terry scores. EIP~\cite{zhang2026rankllm} instead jointly infers question difficulty and model competency from observed successes and failures on a model–question bipartite graph. Complementarily, LoRe~\cite{zhang2025reasoning} studies intrinsic problem complexity, defined by the minimum number of computational steps required to solve a problem.
Among these approaches, IRT remains a common foundation for difficulty-aware evaluation because it provides a principled probabilistic framework for jointly modeling item difficulty and model ability directly from response data~\cite{chen2025item}. 

\section{Theoretical Results and Proofs}
\label{sec:app_proof}

\subsection{Measure Mismatch}
\label{sec:app_proof_mismatch}

We first restate Theorem~\ref{thm:w1_rank_stability} and provide the corresponding proof.

\begin{theorem}[Distribution mismatch controls score distortion and rank stability]
Suppose the performance difference function $\Delta$ is $L$-Lipschitz, i.e.,
\[
|\Delta(z)-\Delta(y)| \le L |z-y|,
\qquad \forall z,y \in [0,1].
\]
Then
\[
\bigl|D_B(\theta',\theta) - D_\mu(\theta',\theta)\bigr|
\le
L\, W_1(\nu_B,\mu).
\]
Consequently, if
\[
|D_\mu(\theta',\theta)| > L\, W_1(\nu_B,\mu),
\]
then
\[
\operatorname{sign}\bigl(D_B(\theta',\theta)\bigr)
=
\operatorname{sign}\bigl(D_\mu(\theta',\theta)\bigr),
\]
so the benchmark-induced ordering agrees with the target ordering, i.e., no rank flip can occur.
\end{theorem}

\begin{proof}
Fix any coupling $\pi \in \Pi(\nu_B,\mu)$. Since $\pi$ has marginals $\nu_B$ and $\mu$, we can write
\[
D_B(\theta',\theta) - D_\mu(\theta',\theta)
=
\int \Delta(z)\, d\nu_B(z) - \int \Delta(y)\, d\mu(y)
=
\int_{[0,1]^2} \bigl(\Delta(z)-\Delta(y)\bigr)\, d\pi(z,y).
\]
Taking absolute values and using the Lipschitz property of $\Delta$ gives
\[
\bigl|D_B(\theta',\theta) - D_\mu(\theta',\theta)\bigr|
\le
\int_{[0,1]^2} |\Delta(z)-\Delta(y)|\, d\pi(z,y)
\le
L \int_{[0,1]^2} |z-y|\, d\pi(z,y).
\]
Since this holds for every coupling $\pi \in \Pi(\nu_B,\mu)$, taking the infimum over $\pi$ yields
\[
\bigl|D_B(\theta',\theta) - D_\mu(\theta',\theta)\bigr|
\le
L\, W_1(\nu_B,\mu).
\]

If $D_\mu(\theta',\theta) > 0$, then
\[
D_B(\theta',\theta)
\ge
D_\mu(\theta',\theta) - \bigl|D_B(\theta',\theta) - D_\mu(\theta',\theta)\bigr|
\ge
D_\mu(\theta',\theta) - L\,W_1(\nu_B,\mu)
>
0.
\]
Hence $D_B(\theta',\theta) > 0$. The case $D_\mu(\theta',\theta) < 0$ is symmetric. Therefore $D_B(\theta',\theta)$ and $D_\mu(\theta',\theta)$ must have the same sign, so no rank flip can occur.
\end{proof}

\paragraph{Finite-sample evaluation under repeated sampling.}
In practice, the conditional mean $s_\theta(z)$ is not directly observed. Instead, benchmark scores are estimated from finitely many repeated samples of model outputs. For the $i$-th benchmark example with difficulty level $z_i$, and for each independent trial $r \in \{1,\dots,k\}$, let
\[
Y_{i,r}^{\theta} \in [0,1]
\]
denote the observed score of model $\theta$ on the $r$-th sample, with
\[
\mathbb{E}\!\left[Y_{i,r}^{\theta}\mid z_i\right] = s_\theta(z_i).
\]
Define the $k$-sample benchmark score as
\[
\widehat{S}_{B,k}(\theta)
:=
\frac{1}{nk}\sum_{i=1}^n \sum_{r=1}^k Y_{i,r}^{\theta},
\]
and the corresponding observed performance difference as
\[
\widehat{D}_{B,k}(\theta',\theta)
:=
\widehat{S}_{B,k}(\theta')-\widehat{S}_{B,k}(\theta).
\]

\begin{theorem}[Finite-sample score distortion and ordering stability under repeated sampling]
\label{thm:noisy_rank_stability}
Suppose the performance difference function $\Delta$ is $L$-Lipschitz. Assume that, conditional on the difficulty levels $z_1,\dots,z_n$, the collection of random variables
\[
\left\{Y_{i,r}^{\theta},\, Y_{i,r}^{\theta'}\right\}_{i=1,\dots,n;\, r=1,\dots,k}
\]
is independent, and that each variable takes values in $[0,1]$. Then for any $\delta \in (0,1)$, with probability at least $1-\delta$,
\[
\bigl|\widehat{D}_{B,k}(\theta',\theta)-D_\mu(\theta',\theta)\bigr|
\le
L\,W_1(\nu_B,\mu)
+
\sqrt{\frac{2\log(2/\delta)}{nk}}.
\]
Consequently, if
\[
|D_\mu(\theta',\theta)|
>
L\,W_1(\nu_B,\mu)
+
\sqrt{\frac{2\log(2/\delta)}{nk}},
\]
then with probability at least $1-\delta$,
\[
\operatorname{sign}\!\bigl(\widehat{D}_{B,k}(\theta',\theta)\bigr)
=
\operatorname{sign}\!\bigl(D_\mu(\theta',\theta)\bigr),
\]
and hence no observed rank flip occurs.
\end{theorem}

\begin{proof}
We first decompose the error as
\[
\widehat{D}_{B,k}(\theta',\theta)-D_\mu(\theta',\theta)
=
\Bigl(\widehat{D}_{B,k}(\theta',\theta)-D_B(\theta',\theta)\Bigr)
+
\Bigl(D_B(\theta',\theta)-D_\mu(\theta',\theta)\Bigr).
\]
Therefore,
\[
\bigl|\widehat{D}_{B,k}(\theta',\theta)-D_\mu(\theta',\theta)\bigr|
\le
\bigl|\widehat{D}_{B,k}(\theta',\theta)-D_B(\theta',\theta)\bigr|
+
\bigl|D_B(\theta',\theta)-D_\mu(\theta',\theta)\bigr|.
\]

The second term is deterministic and, by Theorem~\ref{thm:w1_rank_stability}, satisfies
\[
\bigl|D_B(\theta',\theta)-D_\mu(\theta',\theta)\bigr|
\le
L\,W_1(\nu_B,\mu).
\]

It remains to bound the first term. Define
\[
Z_{i,r}
:=
Y_{i,r}^{\theta'}-Y_{i,r}^{\theta}.
\]
Since $Y_{i,r}^{\theta'}, Y_{i,r}^{\theta}\in[0,1]$, we have
\[
Z_{i,r}\in[-1,1].
\]
Moreover, conditional on $z_i$,
\[
\mathbb{E}[Z_{i,r}\mid z_i]
=
\mathbb{E}[Y_{i,r}^{\theta'}\mid z_i]
-
\mathbb{E}[Y_{i,r}^{\theta}\mid z_i]
=
s_{\theta'}(z_i)-s_\theta(z_i)
=
\Delta(z_i).
\]
Hence
\[
\widehat{D}_{B,k}(\theta',\theta)-D_B(\theta',\theta)
=
\frac{1}{nk}\sum_{i=1}^n \sum_{r=1}^k \Bigl(Z_{i,r}-\Delta(z_i)\Bigr).
\]
Conditional on $z_1,\dots,z_n$, the variables $\{Z_{i,r}\}$ are independent and each lies in $[-1,1]$. Applying Hoeffding's inequality conditionally yields, for any $t>0$,
\[
\mathbb{P}\!\left(
\left.
\bigl|\widehat{D}_{B,k}(\theta',\theta)-D_B(\theta',\theta)\bigr|
\ge t
\ \right|\ z_1,\dots,z_n
\right)
\le
2\exp\!\left(-\frac{nk\,t^2}{2}\right).
\]
Setting
\[
t=\sqrt{\frac{2\log(2/\delta)}{nk}}
\]
gives, with probability at least $1-\delta$,
\[
\bigl|\widehat{D}_{B,k}(\theta',\theta)-D_B(\theta',\theta)\bigr|
\le
\sqrt{\frac{2\log(2/\delta)}{nk}}.
\]
Combining this bound with the deterministic bound above proves
\[
\bigl|\widehat{D}_{B,k}(\theta',\theta)-D_\mu(\theta',\theta)\bigr|
\le
L\,W_1(\nu_B,\mu)
+
\sqrt{\frac{2\log(2/\delta)}{nk}}
\]
with probability at least $1-\delta$.

For the second claim, if
\[
|D_\mu(\theta',\theta)|
>
L\,W_1(\nu_B,\mu)
+
\sqrt{\frac{2\log(2/\delta)}{nk}},
\]
then the above inequality implies that $\widehat{D}_{B,k}(\theta',\theta)$ cannot cross zero, so it must have the same sign as $D_\mu(\theta',\theta)$. Therefore no observed rank flip occurs.
\end{proof}

\paragraph{Uniform distribution as a principled target distribution without prior information.}

The following result formalizes the uniform distribution as a principled default target when no prior knowledge is available about where performance differences arise along the difficulty axis. We consider an agnostic setting in which the signed performance difference $\Delta$ has a fixed average level $\rho$ over $[0,1]$, but its location is otherwise unrestricted. For a candidate target distribution $\mu$, the quantity $V_\rho(\mu)$ measures the worst-case amount of difference retained under $\mu$. The theorem below shows that the uniform distribution is the unique choice that preserves the full average difference $\rho$ in the worst case.

\begin{theorem}[Uniform target as the unique maximin choice under complete uncertainty]
Let
\[
\mathcal F_\rho
:=
\left\{
\Delta:[0,1]\to[-1,1]\ \text{measurable} :
\int_0^1 \Delta(z)\,dz = \rho
\right\},
\qquad
V_\rho(\mu)
:=
\inf_{\Delta\in\mathcal F_\rho}\int \Delta(z)\,d\mu(z),
\]
where $\rho\in(-1,1)$ and $\mu\in\mathcal P([0,1])$. Then
\[
\sup_{\mu\in\mathcal P([0,1])} V_\rho(\mu)=\rho\]
and the unique maximizer is the uniform distribution $\lambda=\mathrm{Unif}[0,1]$.
\end{theorem}

\begin{proof}
Set
\[
\beta := \frac{1+\rho}{2}\in(0,1).
\]
For any measurable function $\Delta:[0,1]\to[-1,1]$, define
\[
G(z):=\frac{\Delta(z)+1}{2}.
\]
Then $G:[0,1]\to[0,1]$ is measurable, and
\[
\int_0^1 G(z)\,dz
=
\frac{1}{2}\int_0^1 \Delta(z)\,dz+\frac{1}{2}.
\]
Therefore, $\Delta \in \mathcal F_\rho$ if and only if
\[
G:[0,1]\to[0,1]
\qquad\text{and}\qquad
\int_0^1 G(z)\,dz=\beta.
\]
Thus, for any $\mu\in\mathcal P([0,1])$,
\[
V_\rho(\mu)
=
\inf_{\Delta\in\mathcal F_\rho}\int \Delta(z)\, d\mu(z)
=
\inf_{G:\,[0,1]\to[0,1],\, \int_0^1 G(z)\,dz=\beta}
\int (2G(z)-1)\, d\mu(z).
\]
Since $\mu$ is a probability measure,
\[
V_\rho(\mu)
=
2 \inf_{G:\,[0,1]\to[0,1],\, \int_0^1 G(z)\,dz=\beta}
\int G(z)\, d\mu(z)-1.
\]
Define
\[
U_\beta(\mu)
:=
\inf_{G:\,[0,1]\to[0,1],\, \int_0^1 G(z)\,dz=\beta}
\int G(z)\, d\mu(z).
\]
Then
\[
V_\rho(\mu)=2U_\beta(\mu)-1.
\]

We now characterize $\sup_{\mu\in\mathcal P([0,1])} U_\beta(\mu)$.

First, if $\mu=\lambda=\mathrm{Unif}[0,1]$, then for every admissible $G$,
\[
\int G(z)\, d\mu(z)
=
\int_0^1 G(z)\,dz
=
\beta,
\]
so
\[
U_\beta(\lambda)=\beta.
\]
Hence
\[
\sup_{\mu\in\mathcal P([0,1])} U_\beta(\mu)\ge \beta.
\]

Next fix any $\mu\in\mathcal P([0,1])$. By the Lebesgue decomposition theorem, we may write
\[
\mu = f(z)\,dz + \mu_s,
\]
where \(f\ge 0\) is integrable and \(\mu_s \perp dz\). Let
\[
m := \int_0^1 f(z)\,dz = \mu_{\mathrm{ac}}([0,1]) \le 1.
\]
Since \(\mu_s\) is singular with respect to Lebesgue measure, there exists a measurable set \(N\subseteq[0,1]\) such that
\[
|N|=0
\qquad\text{and}\qquad
\mu_s([0,1]\setminus N)=0.
\]
Thus, for any admissible \(G\) that vanishes on \(N\), we have
\[
\int G(z)\, d\mu_s(z)=0.
\]
Therefore,
\[
U_\beta(\mu)
\le
\inf_{\substack{0\le G\le 1,\ \int_0^1 G(z)\,dz=\beta,\\ G=0\ \text{on }N}}
\int_0^1 G(z) f(z)\,dz.
\]
Since \(N\) has Lebesgue measure zero, the additional constraint \(G=0\) on \(N\) does not affect the infimum, so
\[
U_\beta(\mu)
\le
\inf_{0\le G\le 1,\ \int_0^1 G(z)\,dz=\beta}
\int_0^1 G(z) f(z)\,dz.
\]
By the bathtub principle, the right-hand side equals
\[
\int_0^\beta f^\uparrow(u)\,du,
\]
where \(f^\uparrow\) is the nondecreasing rearrangement of \(f\). Since \(f^\uparrow\) is nondecreasing and $\beta \in (0,1)$,
\[
\frac{1}{\beta}\int_0^\beta f^\uparrow(u)\,du
\le
\int_0^1 f^\uparrow(u)\,du
=
\int_0^1 f(z)\,dz \le 1.
\]
Hence
\[
U_\beta(\mu)\le \beta.
\]
Since this holds for every \(\mu\in\mathcal P([0,1])\), we conclude that
\[
\sup_{\mu\in\mathcal P([0,1])} U_\beta(\mu)=\beta.
\]
Consequently,
\[
\sup_{\mu\in\mathcal P([0,1])} V_\rho(\mu)
=
2\sup_{\mu\in\mathcal P([0,1])} U_\beta(\mu)-1
=
2\beta-1
=
\rho.
\]

It remains to prove uniqueness. Suppose \(V_\rho(\mu)=\rho\). Then necessarily
\[
U_\beta(\mu)=\beta.
\]
From the argument above, this implies
\[
\int_0^\beta f^\uparrow(u)\,du=\beta
\qquad\text{and}\qquad
\int_0^1 f(z)\,dz=1.
\]
The second condition implies \(\mu_s=0\), so \(\mu\) is absolutely continuous with density \(f\). Moreover,
\[
\frac{1}{\beta}\int_0^\beta f^\uparrow(u)\,du = 1
\]
equals the average of \(f^\uparrow\) over the whole interval \([0,1]\), which is also \(1\). Since \(f^\uparrow\) is nondecreasing, this is possible only if
\[
f^\uparrow(u)=1
\qquad\text{for a.e. }u\in[0,1].
\]
Hence \(f(z)=1\) for a.e. \(z\in[0,1]\), and therefore \(\mu=\lambda\). This proves uniqueness.
\end{proof}

\subsection{Coverage Gaps}
\label{sec:app_proof_missed}

To analyze whether local improvements can be detected from benchmark evaluations, we formulate the problem as a hypothesis test comparing two performance curves before and after improvement. We consider a generic pair of pre- and post-improvement performance curves
\[
s^{-},s^{+}:[0,1]\to[0,1].
\]
Here, $s^{-}(z)$ and $s^{+}(z)$ represent the expected model performance (e.g., pass@1) at difficulty level $z$ before and after improvement, respectively. Their difference is
\[
\Delta(z):=s^{+}(z)-s^{-}(z).
\]

For the \(i\)-th benchmark example and each trial \(r\in\{1,\dots,k\}\), let
\[
Y_{i,r}^{-}\sim \mathrm{Bernoulli}(s^{-}(z_i)),
\qquad
Y_{i,r}^{+}\sim \mathrm{Bernoulli}(s^{+}(z_i)).
\]
We assume that, conditional on the difficulty levels and the performance curves, all variables
\[
\{Y_{i,r}^{-},Y_{i,r}^{+}\}_{i=1,\dots,n;\ r=1,\dots,k}
\]
are mutually independent.
Define
\[
\widehat{S}_{B,k}(s^-)
:=
\frac{1}{nk}\sum_{i=1}^n\sum_{r=1}^k Y_{i,r}^{-},
\qquad
\widehat{S}_{B,k}(s^+)
:=
\frac{1}{nk}\sum_{i=1}^n\sum_{r=1}^k Y_{i,r}^{+},
\]
and
\[
\widehat{D}_{B,k}(s^+,s^-)
:=
\widehat{S}_{B,k}(s^+)-\widehat{S}_{B,k}(s^-).
\]

Fix $w\in(0,1]$ and $a\in(0,1]$. We define the localized alternative class
\[
\mathcal H_1(w,a)
:=
\Bigl\{
(s^-,s^+):
\exists \text{ interval } I\subset[0,1],\ |I|=w,
\ \text{such that } s^{+}(z)-s^{-}(z)\ge a \ \ \forall z\in I
\Bigr\}.
\]

Here, the parameters $(w,a)$ characterize the improvement pattern: 
$w$ represents the width of the difficulty region where the improvement occurs, and $a$ represents the minimum magnitude of improvement within that region.

The null class is
\[
\mathcal H_0
:=
\{(s^-,s^+): s^-=s^+\}.
\]
Thus we consider the testing problem
\[
H_0:\ (s^-,s^+)\in\mathcal H_0
\qquad\text{versus}\qquad
H_1:\ (s^-,s^+)\in\mathcal H_1(w,a).
\]

This formulation captures the question of whether a benchmark can reliably detect localized meaningful improvements of width $w$ and magnitude at least $a$.

A (possibly randomized) test $\phi$ has level at most $\alpha\in(0,1)$ if
\[
\sup_{(s^-,s^+)\in\mathcal H_0}\mathbb E_{s^-,s^+}[\phi]\le \alpha.
\]
Its worst-case power over $\mathcal H_1(w,a)$ is
\[
\inf_{(s^-,s^+)\in\mathcal H_1(w,a)} \mathbb E_{s^-,s^+}[\phi].
\]
We define the minimax worst-case power by
\[
\Pi_\alpha^*(B;w,a,k)
:=
\sup_{\phi:\ \sup_{(s^-,s^+)\in\mathcal H_0}\mathbb E_{s^-,s^+}[\phi]\le \alpha}
\ \inf_{(s^-,s^+)\in\mathcal H_1(w,a)} \mathbb E_{s^-,s^+}[\phi].
\]

\begin{theorem}[Maximum gap as an exact asymptotic minimax detection boundary]
Fix $\alpha\in(0,1)$ and $a\in(0,1]$.

\begin{enumerate}
\item For every $k\in\mathbb N$, if $w\le G(B)$, then
\[
\Pi_\alpha^*(B;w,a,k)=\alpha.
\]

\item If $w>G(B)$, then
\[
\lim_{k\to\infty}\Pi_\alpha^*(B;w,a,k)=1.
\]
\end{enumerate}

Therefore, $G(B)$ is the exact asymptotic minimax detection boundary for interval-localized differences under benchmark $B$.
\end{theorem}

\begin{proof}
We prove the two parts separately.

\smallskip
\noindent\textbf{Part 1: if $w\le G(B)$, then $\Pi_\alpha^*(B;w,a,k)=\alpha$.}

By definition of the maximum gap, there exists an index $j\in\{0,\dots,n\}$ such that
\[
z_{(j+1)}-z_{(j)}\ge w.
\]
Let
\[
J := (z_{(j)},z_{(j+1)}).
\]
Since
\[
|J| = z_{(j+1)} - z_{(j)} \ge w,
\]
we may choose an interval
\[
I \subseteq J
\qquad\text{such that}\qquad
|I| = w .
\]
Choose any \(c\in[0,1-a]\), and define
\[
s^{-,\star}(z)\equiv c,
\qquad
s^{+,\star}(z):=c+a\,\mathbf 1\{z\in I\}.
\]
Since \(c\in[0,1-a]\) and \(a\in(0,1]\), both \(s^{-,\star}\) and \(s^{+,\star}\) take values in \([0,1]\). Moreover,
\[
s^{+,\star}(z)-s^{-,\star}(z)=a
\qquad\forall z\in I,
\]
so
\[
(s^{-,\star},s^{+,\star})\in\mathcal H_1(w,a).
\]

Because the interval \(I\) contains no benchmark point, for every \(i=1,\dots,n\) we have
\[
s^{-,\star}(z_i)=c,
\qquad
s^{+,\star}(z_i)=c.
\]
Therefore, under \((s^{-,\star},s^{+,\star})\), the benchmark observations have exactly the same joint distribution as under the null pair
\[
(\bar s,\bar s)\in\mathcal H_0
\qquad\text{with}\qquad
\bar s(z)\equiv c.
\]
Indeed, at every observed benchmark point \(z_i\), both the pre- and post-improvement Bernoulli success probabilities are equal to \(c\).

Now let $\phi$ be any test of level at most $\alpha$. Then
\[
\mathbb E_{s^{-,\star},s^{+,\star}}[\phi]
=
\mathbb E_{\bar s,\bar s}[\phi]
\le
\sup_{(s^-,s^+)\in\mathcal H_0}\mathbb E_{s^-,s^+}[\phi]
\le \alpha.
\]
Hence
\[
\inf_{(s^-,s^+)\in\mathcal H_1(w,a)}\mathbb E_{s^-,s^+}[\phi]
\le
\mathbb E_{s^{-,\star},s^{+,\star}}[\phi]
\le \alpha.
\]
Taking the supremum over all level-$\alpha$ tests yields
\[
\Pi_\alpha^*(B;w,a,k)\le \alpha.
\]

For the reverse inequality, consider the trivial randomized test that rejects with probability $\alpha$ independently of the data. It has level exactly $\alpha$, and for every alternative its power is also exactly $\alpha$. Therefore
\[
\Pi_\alpha^*(B;w,a,k)\ge \alpha.
\]
Combining the two inequalities, we conclude
\[
\Pi_\alpha^*(B;w,a,k)=\alpha.
\]

\smallskip
\noindent\textbf{Part 2: if $w>G(B)$, then $\lim_{k\to\infty}\Pi_\alpha^*(B;w,a,k)=1$.}

We construct an explicit level-$\alpha$ test whose worst-case power tends to $1$ as $k\to\infty$.

For each benchmark point $z_i$, define the empirical difference
\[
\widehat{\Delta}_{i,k}
:=
\frac{1}{k}\sum_{r=1}^k\bigl(Y_{i,r}^{+}-Y_{i,r}^{-}\bigr).
\]
Consider the max test
\[
\phi_{\max}
:=
\mathbf 1\left\{
\max_{1\le i\le n}\widehat{\Delta}_{i,k}
>
t_{k,\alpha}
\right\},
\qquad
t_{k,\alpha}
:=
\sqrt{\frac{2\log(n/\alpha)}{k}}.
\]

We first verify that $\phi_{\max}$ has level at most $\alpha$. Under any null pair $(s,s)\in\mathcal H_0$, for each fixed $i$ and $r$,
\[
Y_{i,r}^{+}-Y_{i,r}^{-}\in[-1,1],
\qquad
\mathbb E_{s,s}[Y_{i,r}^{+}-Y_{i,r}^{-}]=0,
\]
and these random variables are independent across $r$. Hence, by Hoeffding's inequality,
\[
\mathbb P_{s,s}\!\left(\widehat{\Delta}_{i,k}>t\right)
\le
\exp\!\left(-\frac{k t^2}{2}\right).
\]
Applying the union bound over $i=1,\dots,n$ gives
\[
\mathbb P_{s,s}(\phi_{\max}=1)
=
\mathbb P_{s,s}\!\left(\max_{1\le i\le n}\widehat{\Delta}_{i,k}>t_{k,\alpha}\right)
\le
\sum_{i=1}^n
\mathbb P_{s,s}\!\left(\widehat{\Delta}_{i,k}>t_{k,\alpha}\right)
\le
n\exp\!\left(-\frac{k t_{k,\alpha}^2}{2}\right)
=
\alpha.
\]
Since this bound is uniform over $(s,s)\in\mathcal H_0$, the test $\phi_{\max}$ has level at most $\alpha$.

Next fix any alternative pair $(s^-,s^+)\in\mathcal H_1(w,a)$. By definition, there exists an interval $I\subset[0,1]$ with $|I|=w$ such that
\[
s^+(z)-s^-(z)\ge a,
\qquad \forall z\in I.
\]
Because $w>G(B)$, no interval of length $w$ can be disjoint from all benchmark points. Therefore, there exists at least one benchmark index $i^\star$ such that
\[
z_{i^\star}\in I.
\]
For this index we have
\[
\mathbb E_{s^-,s^+}[\widehat{\Delta}_{i^\star,k}]
=
s^+(z_{i^\star})-s^-(z_{i^\star})
\ge a.
\]

Again by Hoeffding's inequality, for any $t\in\mathbb R$,
\[
\mathbb P_{s^-,s^+}\!\left(
\widehat{\Delta}_{i^\star,k}\le t
\right)
\le
\exp\!\left(
-\frac{k}{2}\bigl(a-t\bigr)_+^2
\right),
\]
where $(u)_+:=\max\{u,0\}$. Setting $t=t_{k,\alpha}$ yields
\[
\mathbb P_{s^-,s^+}\!\left(
\widehat{\Delta}_{i^\star,k}\le t_{k,\alpha}
\right)
\le
\exp\!\left(
-\frac{k}{2}\bigl(a-t_{k,\alpha}\bigr)_+^2
\right).
\]
Since the event
\[
\left\{\widehat{\Delta}_{i^\star,k}>t_{k,\alpha}\right\}
\]
implies
\[
\left\{\max_{1\le i\le n}\widehat{\Delta}_{i,k}>t_{k,\alpha}\right\},
\]
we obtain
\[
\mathbb P_{s^-,s^+}(\phi_{\max}=1)
\ge
\mathbb P_{s^-,s^+}\!\left(
\widehat{\Delta}_{i^\star,k}>t_{k,\alpha}
\right)
\ge
1-
\exp\!\left(
-\frac{k}{2}\bigl(a-t_{k,\alpha}\bigr)_+^2
\right).
\]
This lower bound is uniform over all $(s^-,s^+)\in\mathcal H_1(w,a)$, so
\[
\inf_{(s^-,s^+)\in\mathcal H_1(w,a)}
\mathbb E_{s^-,s^+}[\phi_{\max}]
\ge
1-
\exp\!\left(
-\frac{k}{2}\bigl(a-t_{k,\alpha}\bigr)_+^2
\right).
\]
Since the minimax power is at least the worst-case power of any specific level-$\alpha$ test,
\[
\Pi_\alpha^*(B;w,a,k)
\ge
1-
\exp\!\left(
-\frac{k}{2}\bigl(a-t_{k,\alpha}\bigr)_+^2
\right).
\]
Finally, because
\[
t_{k,\alpha}
=
\sqrt{\frac{2\log(n/\alpha)}{k}}
\longrightarrow 0
\qquad\text{as }k\to\infty,
\]
the right-hand side converges to $1$. Since $\Pi_\alpha^*(B;w,a,k)\le 1$ by definition of power, the squeeze theorem implies
\[
\lim_{k\to\infty}\Pi_\alpha^*(B;w,a,k)=1.
\]

This completes the proof.
\end{proof}

\section{Difficulty Estimation Details and Additional Results}
\label{sec:app_diff_est_details}
\subsection{Benchmark Details}
\label{sec:app_benchmark_details}
Tab.~\ref{tab:benchmark_summary} summarizes the benchmarks used in our analysis, including their domains, the number of examples, and the number of rollouts per item during evaluation.

\begin{table}[h]
\centering
\caption{\textbf{Summary of the benchmarks used in our analysis.} For compactness, we abbreviate SWE-bench Verified (Mini) as SWE-bench-Mini, and BIG-Bench Hard-Logical Deduction-\{3,5,7\} objects as BBH-\{3,5,7\} objects throughout the paper.}
\label{tab:benchmark_summary}
\resizebox{1.0\linewidth}{!}{
\begin{tabular}{l l r r}
\toprule
\textbf{Benchmark} & \textbf{Domain} & \textbf{\# Examples} & \textbf{\# Rollouts / Item} \\
\midrule
MATH-500~\cite{lightman2023let} & Math & 500 & 8 \\
Minerva-Math~\cite{lewkowycz2022solving}  & Math & 272 & 8 \\
OlympiadBench~\cite{he2024olympiadbench} & Math & 675 & 8 \\
Gaokao2023en~\cite{zhang2023evaluating} & Math & 385 & 8 \\
AMC23~\cite{amc23} & Math & 40 & 32 \\
AIME24~\cite{AIME24} & Math & 30 & 32 \\
AIME25~\cite{AIME25} & Math & 30 & 32 \\
AIME26~\cite{aime26} & Math & 30 & 32 \\
\midrule
HumanEval~\cite{chen2021codex} & Coding & 164 & 4 \\
MBPP+~\cite{liu2023your} & Coding & 378 & 4 \\
\midrule
GAIA~\cite{mialon2023gaia} & Agentic Tasks & 165 & 1 \\
CORE-Bench Hard
~\cite{siegel2024core} & Agentic Tasks & 45 & 1 \\
SWE-bench Verified (Mini)~\cite{jimenez2023swe} & Agentic Tasks & 50 & 1 \\
USACO
~\cite{shi2024can} & Agentic Tasks & 307 & 1 \\
\midrule
MMLU-Pro-Biology~\cite{wang2024mmlu} & Science & 717 & 4 \\
MMLU-Pro-Chemistry~\cite{wang2024mmlu} & Science & 1132 & 4 \\
MMLU-Pro-Physics~\cite{wang2024mmlu} & Science & 1299 & 4 \\
GPQA-Diamond~\cite{rein2024gpqa} & Science & 198 & 4 \\
ScienceQA~\cite{lu2022learn} & Science & 4241 & 4 \\
\midrule
IFEval~\cite{zhou2023instruction} & Instruction Following & 541 & 4 \\
\midrule
BIG-Bench Hard-Logical Deduction-3 objects~\cite{suzgun2022challenging} & Logic & 250  & 4 \\
BIG-Bench Hard-Logical Deduction-5 objects~\cite{suzgun2022challenging} & Logic & 250  & 4 \\
BIG-Bench Hard-Logical Deduction-7 objects~\cite{suzgun2022challenging} & Logic & 250 & 4 \\
ZebraLogic~\cite{zebralogic2025} & Logic & 1000 & 4 \\
\midrule
MedQA~\cite{jin2021disease} & Medicine & 1273 & 4 \\
\bottomrule
\end{tabular}
}
\end{table}

\subsection{Model Selection Details}
\label{sec:app_model_details}

For non-agentic tasks, we select models based on the overall leaderboard of Text Arena, whose evaluation suite covers a diverse set of domains including \texttt{Expert}, \texttt{Hard Prompts}, \texttt{Coding}, \texttt{Math}, \texttt{Creative Writing}, \texttt{Instruction Following}, and \texttt{Longer Query}. The overall score therefore provides a reasonable proxy for model capability across a broad range of tasks. To obtain models with diverse capability levels, we partition the leaderboard into 10 equal-width tiers according to the overall score. We then randomly select one model from each tier and additionally include one model from both the top and bottom of the leaderboard. This results in a total of 12 models spanning a wide spectrum of capabilities. The detailed model information is summarized in Tab.~\ref{tab:model_details}.

\begin{table}[h]
\centering
\caption{\textbf{Summary of selected models for non-agentic tasks.}}
\label{tab:model_details}
\resizebox{0.85\linewidth}{!}{
\begin{tabular}{l l l c c}
\toprule
\textbf{Model} & \textbf{Size} & \textbf{Provider} & \textbf{Temperature} & \textbf{Top-$p$} \\
\midrule
Qwen3.5-397B-A17B & 397B & Alibaba Qwen & 0.6 & 0.95 \\
GLM-4.6  & 357B & Zhipu AI & 1.0 & 0.95 \\
GLM-4.5-Air & 110B & Zhipu AI & 0.6 & 0.95 \\
Qwen3-Next-80B-A3B-Instruct  & 80B & Alibaba Qwen & 0.7 & 0.8 \\
QwQ-32B & 32B & Alibaba Qwen & 0.6 & 0.95 \\
Olmo-3.1-32B-Think & 32B & Allen AI & 0.6 & 0.95 \\
gpt-oss-20b & 20B & OpenAI & 1.0 & 1.0 \\
phi-4 & 14B & Microsoft & 0.8 & 0.95 \\
Llama-3.1-8B-Instruct & 8B & Meta-Llama & 0.6 & 0.9 \\
vicuna-7b-v1.5 & 7B & LMSYS & 0.9 & 0.6 \\
Llama-3.2-3B-Instruct & 3B & Meta-Llama & 0.6 & 0.9 \\
SmolLM2-1.7B-Instruct & 1.7B & HuggingFaceTB & 0.7 & 0.8 \\
\bottomrule
\end{tabular}
}
\end{table}

For agentic tasks, we use models included in the HAL evaluation platform. These include several frontier proprietary models from major providers, such as Claude Sonnet~4.5, Claude Opus~4.1, GPT-5, and Gemini 2.0 Flash.

\subsection{IRT Fitting and Diagnostics}
\label{sec:app_irt}

To obtain reliable difficulty estimates, we conduct several diagnostic checks on both the response matrices and the IRT fitting procedure.

\paragraph{Response matrix diagnostics.}
We examine each response matrix using four diagnostics. The first two, \emph{model information} and \emph{item information}, measure whether the matrix contains enough discriminative signal along the model and item dimensions~\cite{baker2001basics}. 
The latter two, \emph{model inversion rate} and \emph{item inversion rate}, measure how well the empirical response matrix agrees with the monotone structure implied by 1PL-IRT~\cite{sijtsma1998nonparametric, sijtsma1996survey}. 

Formally, for each item--model pair $(i,j)$, we aggregate the observed responses into the empirical success rate of model $j$ on item $i$, defined as $p_{ij}=n_{ij}^{\mathrm{correct}}/n_{ij}^{\mathrm{trials}}$. We then aggregate these success rates along the model and item dimensions:
\[
p_j=\frac{1}{|\mathcal I|}\sum_{i\in\mathcal I} p_{ij}, \qquad
p_i=\frac{1}{|\mathcal J|}\sum_{j\in\mathcal J} p_{ij},
\]
where $\mathcal I$ and $\mathcal J$ denote the sets of items and models, respectively. We define model information as the average of $4p_j(1-p_j)$ over models, and item information as the average of $4p_i(1-p_i)$ over items. The intuition is that if many models solve nearly all items or almost none of them, then their aggregate success rates are saturated and contribute limited information for distinguishing model abilities. The same holds when many items are solved by nearly all models or by almost none. In contrast, aggregate success rates near $0.5$ are more informative because they correspond to models or items that are neither saturated nor degenerate. The factor $4$ rescales this quantity to $[0,1]$.

To evaluate whether the response matrix is consistent with the ordering implied by 1PL-IRT, we first rank models by overall success rate and items by overall solve rate. We then compute two inversion rates. The model inversion rate is the fraction of ordered model pairs, evaluated item by item, for which a stronger model performs worse than a weaker one. The item inversion rate is the fraction of ordered item pairs, evaluated model by model, for which an easier item is solved less often than a harder one. Lower inversion rates indicate better agreement with the latent ordering assumed by IRT.

In our experiments, we use model information $\ge 0.3$, item information $\ge 0.3$, model inversion rate $\le 0.3$, and item inversion rate $\le 0.3$ as basic thresholds. Passing these checks provides evidence that the matrix contains sufficient signal and is reasonably consistent with the ordering assumptions of IRT.

\paragraph{IRT fitting.} We fit the 1PL-IRT model using Bayesian inference. For multiple-choice benchmarks, we additionally include a guessing parameter:
\[
\Pr(Y_{ij}=1)=g_i+(1-g_i)\sigma(\theta_j-b_i),
\]
where \(g_i\) denotes the chance-level probability of answering item \(i\) correctly by guessing.

Model abilities $\theta_j$ and item difficulties $b_i$ are assigned independent Gaussian priors,
\[
\theta_j \sim \mathcal{N}(0,1), \qquad
b_i \sim \mathcal{N}(0,1).
\]
Because each model–item pair is evaluated multiple times, we use a binomial likelihood to model the number of correct responses. Posterior inference is performed using MCMC with the No-U-Turn Sampler (NUTS) implemented in PyMC~\cite{pymc2023}. We run four chains with 1{,}000 warmup iterations and 1{,}000 posterior samples per chain. We set the target acceptance rate to 0.95.

We also explored alternative prior structures to assess robustness. In particular, we considered a hierarchical prior for model abilities,
\[
\theta_j \sim \mathcal{N}(\mu_\theta,\sigma_\theta^2), \qquad
\mu_\theta \sim \mathcal{N}(0,1), \quad
\sigma_\theta \sim \text{HalfNormal}(1),
\]
while keeping item difficulties under the same Gaussian prior as above. We additionally tested a hierarchical prior for item difficulties,
\[
b_i \sim \mathcal{N}(\mu_b,\sigma_b^2), \qquad
\mu_b \sim \mathcal{N}(0,1), \quad
\sigma_b \sim \text{HalfNormal}(1),
\]
with model abilities using the standard Gaussian prior. Across these variants, the estimated item difficulties and model abilities were highly consistent. Therefore, we report results using the simplest specification with independent $\mathcal{N}(0,1)$ priors.

\paragraph{IRT diagnostics.}

To assess convergence and sampling quality, we monitor $3$ standard MCMC diagnostics: the maximum Gelman–Rubin statistic $\hat{R}$ across parameters, the minimum bulk effective sample size (ESS), and the number of divergent transitions. The $\hat{R}$ statistic measures agreement across chains, with values close to 1 indicating good convergence. Bulk ESS quantifies the effective number of independent posterior samples in the main mass of the distribution, with larger values indicating more reliable estimation. Divergent transitions indicate potential numerical or geometric pathologies during NUTS sampling, and zero divergent transitions are preferred.

\paragraph{Diagnostic results.}
Tabs.~\ref{tab:response_matrix_diagnostics} and~\ref{tab:mcmc_diagnostics} summarize the diagnostic results for all benchmarks. All benchmarks pass the diagnostic thresholds.

\begin{table}[h]
\centering
\caption{\textbf{Response matrix diagnostic results.} $\uparrow$ indicates that larger values are preferred and $\downarrow$ indicates that smaller values are preferred. We use model information $\geq 0.3$, item information $\geq 0.3$, model inversion rate $\leq 0.3$, and item inversion rate $\leq 0.3$ as basic thresholds for response matrix diagnostics. All benchmarks pass these checks.}
\label{tab:response_matrix_diagnostics}
\resizebox{1.0\linewidth}{!}{
\begin{tabular}{lrrrr}
\toprule
\textbf{Benchmark} & \textbf{Model Information $\uparrow$} & \textbf{Item Information $\uparrow$} & \textbf{Model Inversion Rate $\downarrow$} & \textbf{Item Inversion Rate $\downarrow$} \\
\midrule
AMC23 & 0.458 & 0.867 & 0.062 & 0.177 \\
AIME24 & 0.514 & 0.824 & 0.040 & 0.101 \\
AIME25 & 0.515 & 0.806 & 0.057 & 0.118 \\
AIME26 & 0.465 & 0.831 & 0.046 & 0.089 \\
Gaokao2023en & 0.661 & 0.681 & 0.101 & 0.112 \\
HumanEval & 0.546 & 0.675 & 0.165 & 0.205 \\
MATH-500 & 0.463 & 0.730 & 0.063 & 0.125 \\
MBPP+ & 0.665 & 0.636 & 0.226 & 0.235 \\
Minerva-Math & 0.724 & 0.554 & 0.174 & 0.107 \\
OlympiadBench & 0.715 & 0.663 & 0.106 & 0.093 \\
BBH-3 objects & 0.344 & 0.541 & 0.069 & 0.231 \\
BBH-5 objects & 0.421 & 0.735 & 0.076 & 0.292 \\
BBH-7 objects & 0.451 & 0.790 & 0.080 & 0.285 \\
GPQA-Diamond & 0.843 & 0.852 & 0.249 & 0.251 \\
IFEval & 0.663 & 0.573 & 0.217 & 0.162 \\
MedQA & 0.795 & 0.777 & 0.264 & 0.262 \\
MMLU-Pro-Biology & 0.617 & 0.632 & 0.172 & 0.160 \\
MMLU-Pro-Chemistry & 0.676 & 0.787 & 0.162 & 0.186 \\
MMLU-Pro-Physics & 0.687 & 0.787 & 0.176 & 0.195 \\
ScienceQA & 0.659 & 0.434 & 0.184 & 0.102 \\
ZebraLogic & 0.581 & 0.796 & 0.048 & 0.060 \\
GAIA & 0.879 & 0.739 & 0.251 & 0.178\\
CORE-Bench Hard & 0.651 & 0.578 & 0.192 & 0.135 \\
SWE-bench-Mini & 0.709 & 0.674 & 0.112 & 0.103\\
USACO & 0.908 & 0.472 & 0.205 & 0.059\\
\bottomrule
\end{tabular}
}
\end{table}

\begin{table}[h]
\centering
\caption{\textbf{MCMC diagnostic results for IRT fitting.} $\uparrow$ indicates that larger values are preferred and $\downarrow$ indicates that smaller values are preferred.}
\label{tab:mcmc_diagnostics}
\resizebox{0.7\linewidth}{!}{
\begin{tabular}{lrrr}
\toprule
\textbf{Benchmark} & \textbf{Max $\hat{R}$ $\downarrow$} & \textbf{Min Bulk ESS $\uparrow$} & \textbf{\# Divergent Transitions $\downarrow$} \\
\midrule
AMC23 & 1.01 & 334 & 0 \\
AIME24 & 1.02 & 272 & 0 \\
AIME25 & 1.01 & 246 & 0 \\
AIME26 & 1.00 & 385 & 0 \\
Gaokao2023en & 1.01 & 569 & 0 \\
HumanEval & 1.01 & 826 & 0 \\
MATH-500 & 1.01 & 1296 & 0\\
MBPP+ & 1.01 & 1167 & 0 \\
Minerva-Math & 1.01 & 487 & 0\\
OlympiadBench & 1.01 & 1246 & 0\\
BBH-3 objects & 1.00 & 1996 & 0 \\
BBH-5 objects & 1.01 & 1450 & 0 \\
BBH-7 objects & 1.00 & 1218 & 0 \\
GPQA-Diamond & 1.00 & 1437 & 0 \\
IFEval & 1.01 & 2120 & 0 \\
MedQA & 1.01 & 1733 & 0 \\
MMLU-Pro-Biology & 1.01 & 2511 & 0 \\
MMLU-Pro-Chemistry & 1.01 & 2014 & 0 \\
MMLU-Pro-Physics & 1.01 & 2002 & 0 \\
ScienceQA & 1.01 & 1757 & 0 \\
ZebraLogic & 1.01 & 1960 & 0 \\
GAIA & 1.00 & 1861 & 0\\
CORE-Bench Hard & 1.01 & 3272 & 0\\
SWE-bench-Mini & 1.00 & 4419 & 0\\
USACO & 1.01 & 2102 & 0\\
\bottomrule
\end{tabular}
}
\end{table}

\paragraph{Leave-one-out robustness analysis.}
For each non-agentic benchmark, we assess the robustness of IRT-based difficulty estimation by performing a leave-one-out analysis over the selected model set. Specifically, given the $12$ selected models, we refit IRT $12$ times, each time removing one model and re-estimating item difficulties from the remaining $11$ models. We then compare the resulting difficulty estimates with the original estimates obtained from all $12$ models using Pearson and Spearman correlations. For each benchmark, we report both the average and the minimum correlation across the $12$ leave-one-out refits. As shown in Tab.~\ref{tab:loo_irt_robustness}, the average Pearson and Spearman correlations are at least $0.9723$ and $0.9704$, respectively, across all benchmarks. Even the worst-case leave-one-out correlation remains high, with the minimum Pearson and Spearman correlations being $0.9137$ and $0.9089$, respectively. These results indicate that the estimated item difficulties are stable to the exclusion of any single model, supporting the adequacy of our $12$-model selection for stable difficulty estimation.

\begin{table*}[t]
\centering
\small
\caption{\textbf{Leave-one-out robustness of IRT difficulty estimation.}
For each benchmark, we fit IRT using the full set of 12 models, then refit IRT after removing one model at a time and compare the resulting item difficulties with those from the full fit. We report the mean and minimum Pearson and Spearman correlations across the 12 leave-one-out refits. The consistently high correlations indicate that our selected 12 models are sufficient for stable item difficulty estimation.}
\label{tab:loo_irt_robustness}
\resizebox{1.0\textwidth}{!}{
\begin{tabular}{lrrrr}
\toprule
\textbf{Benchmark} & \textbf{Pearson Mean} & \textbf{Pearson Min} & \textbf{Spearman Mean} & \textbf{Spearman Min} \\
\midrule
AMC23 & 0.9962 & 0.9820 & 0.9921 & 0.9484 \\
AIME24 & 0.9985 & 0.9940 & 0.9977 & 0.9920 \\
AIME25 & 0.9978 & 0.9904 & 0.9941 & 0.9809 \\
AIME26 & 0.9979 & 0.9937 & 0.9972 & 0.9888 \\
Gaokao2023en & 0.9974 & 0.9953 & 0.9957 & 0.9905 \\
HumanEval & 0.9902 & 0.9771 & 0.9871 & 0.9735 \\
MATH-500 & 0.9952 & 0.9896 & 0.9935 & 0.9844 \\
MBPP+ & 0.9902 & 0.9751 & 0.9840 & 0.9599 \\
Minerva-Math  & 0.9974 & 0.9951 & 0.9896 & 0.9871 \\
OlympiadBench & 0.9980 & 0.9953 & 0.9973 & 0.9946 \\
BBH-3 objects & 0.9766 & 0.9137 & 0.9764 & 0.9089 \\
BBH-5 objects & 0.9723 & 0.9207 & 0.9704 & 0.9162 \\
BBH-7 objects & 0.9764 & 0.9360 & 0.9745 & 0.9248 \\
GPQA-Diamond & 0.9941 & 0.9893 & 0.9937 & 0.9852 \\
IFEval & 0.9945 & 0.9853 & 0.9910 & 0.9806 \\
MedQA & 0.9916 & 0.9875 & 0.9861 & 0.9722 \\
MMLU-Pro-Biology & 0.9951 & 0.9894 & 0.9923 & 0.9792 \\
MMLU-Pro-Chemistry & 0.9955 & 0.9914 & 0.9936 & 0.9860 \\
MMLU-Pro-Physics & 0.9952 & 0.9907 & 0.9934 & 0.9839 \\
ScienceQA & 0.9965 & 0.9893 & 0.9869 & 0.9330 \\
ZebraLogic  & 0.9974 & 0.9948 & 0.9966 & 0.9909 \\
\midrule
\textbf{Average} & \textbf{0.9926} & \textbf{0.9798} & \textbf{0.9897} & \textbf{0.9696} \\
\bottomrule
\end{tabular}
}
\end{table*}

\subsection{Metric Calculation Details}

For a benchmark with $n$ items, we compare its normalized difficulty points (min-max normalization) to a same-size empirical uniform reference on $[0,1]$. Let
$z_{(1)}\le \cdots \le z_{(n)}$ denote the sorted normalized difficulties. We define the uniform reference points as the midpoints of $n$ equal-width bins:
\[
u_i=\frac{i-\frac{1}{2}}{n}, \qquad i=1,\ldots,n.
\]
The $\mathrm{DDM}$ is then computed as
\[
\mathrm{DDM}(B)
=
\frac{1}{n}\sum_{i=1}^n |z_{(i)}-u_i|.
\]
This is the one-dimensional 1-Wasserstein distance between the benchmark empirical measure and the same-size empirical uniform reference
\[
\lambda_n=\frac{1}{n}\sum_{i=1}^n \delta_{u_i}.
\]
For the gap-based metric, let
$z_{(0)}=0$ and $z_{(n+1)}=1$. The maximum gap is
\[
G(B)=\max_{0\le i\le n}\bigl(z_{(i+1)}-z_{(i)}\bigr).
\]

\subsection{Model Comparison for Different Difficulty Estimation Models}
\label{app:irt_model_comparison}

Our empirical analysis relies on a scalar 1PL IRT difficulty estimate. To assess whether this modeling choice is supported by the response data, we compare 1PL against more flexible IRT specifications, including 2PL and multidimensional IRT models.

We first compare 1PL and 2PL using Bayesian Information Criterion (BIC), where lower BIC is preferred. As shown in Tab.~\ref{tab:irt_bic_1pl_2pl}, BIC selects 1PL for 21 out of the 25 benchmarks, often by large margins. This indicates that for most benchmarks, adding item-specific discrimination parameters is not supported after accounting for model complexity.

For the remaining four benchmarks where BIC favors 2PL, the 2PL fits exhibit problematic sampling diagnostics, while the corresponding 1PL fits remain well-behaved. Tab.~\ref{tab:irt_2pl_diagnostics} reports the convergence diagnostics for these cases. For example, on Gaokao2023en, the 2PL fit produces max $\hat{R}=1.06$, min bulk ESS $=63$, and 983 divergent transitions, whereas the 1PL fit has max $\hat{R}=1.01$, min bulk ESS $=569$, and zero divergent transitions. Similar issues appear for AIME24, AIME25, and AMC23. Therefore, although 2PL is occasionally favored by BIC, these 2PL fits do not provide reliable alternative difficulty estimates.

\begin{table}[t]
\centering
\small
\caption{\textbf{BIC comparison between 1PL and 2PL across all benchmarks.}
Lower BIC is preferred. We define $\Delta\mathrm{BIC}=\mathrm{BIC}_{\mathrm{2PL}}-\mathrm{BIC}_{\mathrm{1PL}}$, so positive values favor 1PL. BIC selects 1PL for 21 out of 25 benchmarks.}
\label{tab:irt_bic_1pl_2pl}
\begin{tabular}{lrrrr}
\toprule
\textbf{Benchmark} & \textbf{$\mathrm{BIC}_{\mathrm{1PL}}$} & \textbf{$\mathrm{BIC}_{\mathrm{2PL}}$} & \textbf{$\Delta\mathrm{BIC}$} & \textbf{Selected} \\
\midrule
ScienceQA & 126817 & 151726 & +24909 & 1PL \\
MMLU-Pro-Physics & 44290 & 52569 & +8279 & 1PL \\
MedQA & 52414 & 60340 & +7926 & 1PL \\
ZebraLogic & 22000 & 29319 & +7319 & 1PL \\
MMLU-Pro-Chemistry & 37446 & 44645 & +7199 & 1PL \\
MMLU-Pro-Biology & 21712 & 25399 & +3687 & 1PL \\
IFEval & 17614 & 20286 & +2672 & 1PL \\
USACO & 5349 & 7500 & +2151 & 1PL \\
BBH-3 objects & 5362 & 6958 & +1596 & 1PL \\
MBPP+ & 14222 & 15778 & +1556 & 1PL \\
BBH-7 objects & 6838 & 8260 & +1422 & 1PL \\
BBH-5 objects & 6515 & 7910 & +1395 & 1PL \\
GAIA & 7130 & 8166 & +1036 & 1PL \\
Minerva-Math & 11865 & 12798 & +933 & 1PL \\
GPQA-Diamond & 8080 & 8739 & +659 & 1PL \\
MATH-500 & 17194 & 17739 & +545 & 1PL \\
OlympiadBench & 25166 & 25466 & +300 & 1PL \\
SWE-bench-Mini & 1624 & 1876 & +252 & 1PL \\
CORE-Bench Hard & 2881 & 3106 & +225 & 1PL \\
HumanEval & 5188 & 5376 & +188 & 1PL \\
AIME26 & 1358 & 1362 & +4 & 1PL \\
\midrule
Gaokao2023en & 15054 & 14977 & -77 & 2PL \\
AIME25 & 1593 & 1512 & -81 & 2PL \\
AIME24 & 1496 & 1409 & -87 & 2PL \\
AMC23 & 2607 & 2052 & -555 & 2PL \\
\bottomrule
\end{tabular}
\end{table}

\begin{table}[t]
\centering
\small
\caption{\textbf{Posterior diagnostics for benchmarks where BIC favors 2PL.}
We report the maximum $\hat{R}$, minimum bulk ESS, and number of divergent transitions across fitted parameters. Although BIC favors 2PL on these four benchmarks, the 2PL fits exhibit substantially worse sampling diagnostics, while the corresponding 1PL fits remain well-behaved.}
\label{tab:irt_2pl_diagnostics}
\resizebox{\textwidth}{!}{
\begin{tabular}{lrrrrrr}
\toprule
\multirow{2}{*}{\textbf{Benchmark}} 
& \multicolumn{3}{c}{\textbf{1PL}} 
& \multicolumn{3}{c}{\textbf{2PL}} \\
\cmidrule(lr){2-4} \cmidrule(lr){5-7}
& \textbf{Max $\hat{R}$ $\downarrow$} 
& \textbf{Min Bulk ESS $\uparrow$} 
& \textbf{\# Div. $\downarrow$} 
& \textbf{Max $\hat{R}$ $\downarrow$} 
& \textbf{Min Bulk ESS $\uparrow$} 
& \textbf{\# Div. $\downarrow$} \\
\midrule
Gaokao2023en & 1.01 & 569 & 0 & 1.06 & 63  & 983 \\
AIME25       & 1.01 & 246  & 0 & 1.02 & 188 & 18  \\
AIME24       & 1.02 & 272  & 0 & 1.04 & 118 & 2   \\
AMC23        & 1.01 & 334  & 0 & 1.03 & 146 & 8   \\
\bottomrule
\end{tabular}
}
\end{table}

We also evaluate multidimensional IRT models with two and three latent dimensions, denoted MIRT-2 and MIRT-3. Tab.~\ref{tab:mirt_diagnostics} reports convergence diagnostics on the 15 smaller benchmarks. Overall, these multidimensional fits do not provide stable posterior estimates: most runs have maximum $\hat{R}$ well above the standard convergence threshold of 1.05 and extremely small minimum bulk ESS, indicating poor mixing. On larger benchmarks, the multidimensional models are substantially more expensive and did not complete reliably within our fitting budget under the same Bayesian fitting procedure. Thus, the current response matrices do not support stable multidimensional difficulty estimates for our empirical analysis.

\begin{table}[t]
\centering
\small
\caption{\textbf{Convergence diagnostics for multidimensional IRT fits on smaller benchmarks.}
We report the maximum $\hat{R}$, minimum bulk ESS, and number of divergent transitions across fitted parameters. The multidimensional fits generally exhibit poor mixing, with high maximum $\hat{R}$ and very small minimum bulk ESS, indicating that they do not provide stable alternative difficulty estimates.}
\label{tab:mirt_diagnostics}
\resizebox{\textwidth}{!}{
\begin{tabular}{lrrrrrr}
\toprule
\multirow{2}{*}{\textbf{Benchmark}}
& \multicolumn{3}{c}{\textbf{MIRT-2}}
& \multicolumn{3}{c}{\textbf{MIRT-3}} \\
\cmidrule(lr){2-4} \cmidrule(lr){5-7}
& \textbf{Max $\hat{R}$ $\downarrow$}
& \textbf{Min Bulk ESS $\uparrow$}
& \textbf{\# Div. $\downarrow$}
& \textbf{Max $\hat{R}$ $\downarrow$}
& \textbf{Min Bulk ESS $\uparrow$}
& \textbf{\# Div. $\downarrow$} \\
\midrule
AIME24 & 1.02 & 202 & 0 & 1.68 & 6  & 0 \\
AIME25 & 1.75 & 6   & 0 & 1.02 & 172 & 0 \\
AIME26 & 1.54 & 7   & 0 & 1.44 & 8  & 0 \\
AMC23 & 1.75 & 6   & 0 & 1.35 & 10 & 0 \\
Minerva-Math & 1.73 & 6 & 0 & 1.37 & 10 & 0 \\
GPQA-Diamond & 1.81 & 6 & 0 & 1.67 & 6 & 0 \\
HumanEval & 1.58 & 7 & 0 & 1.90 & 6 & 0 \\
MBPP+ & 2.24 & 5 & 0 & 1.88 & 6 & 0 \\
BBH-3 objects & 1.78 & 6 & 0 & 1.26 & 12 & 0 \\
BBH-5 objects & 1.92 & 6 & 0 & 1.23 & 13 & 0 \\
BBH-7 objects & 1.84 & 6 & 0 & 1.16 & 20 & 0 \\
CORE-Bench Hard & 1.77 & 6 & 0 & 1.09 & 31 & 0 \\
GAIA & 1.62 & 7 & 0 & 1.27 & 12 & 0 \\
SWE-bench-Mini & 1.48 & 8 & 0 & 1.13 & 20 & 0 \\
USACO & 1.52 & 7 & 0 & 1.35 & 10 & 0 \\
\bottomrule
\end{tabular}
}
\end{table}

Overall, these model-comparison results support the use of 1PL as a reliable difficulty estimation model for our study. While more complex IRT specifications may be useful in other settings, our diagnostics indicate that 1PL provides the most stable and well-supported difficulty estimates for the response matrices analyzed here.

\subsection{Posterior Uncertainty Analysis}
\label{app:metric_uncertainty}

We further quantify how posterior uncertainty in the fitted IRT difficulty estimates affects our calculated metrics. For each posterior draw of the fitted 1PL item difficulties, we recompute $\mathrm{DDM}$ and $\mathrm{Max\ Gap}$, and summarize the posterior variability of these metrics across draws.

\paragraph{$\mathrm{DDM}$ is stable under IRT posterior uncertainty.} Across the 25 benchmarks, the posterior standard deviation of $\mathrm{DDM}$ is typically about 0.011, with range 0.006--0.020. This is small relative to the $\mathrm{DDM}$ point estimates, corresponding to approximately 8\%--12\% in typical cases. Thus, the $\mathrm{DDM}$ values reported in the main paper are not artifacts of difficulty-estimation noise.

\paragraph{$\mathrm{Max\ Gap}$ uncertainty is larger but the coverage-gap conclusion remains robust.}
$\mathrm{Max\ Gap}$ has larger posterior uncertainty, as expected, because it is an extreme-order statistic~\citep{david2004order}: it is defined by the largest adjacent spacing between item difficulties and is therefore more sensitive to local posterior perturbations. Its posterior standard deviation ranges from 0.017 to 0.041. However, the qualitative coverage-gap conclusion remains robust. To verify this, we compare the observed $\mathrm{Max\ Gap}$ against a same-size uniform-sampling baseline under posterior difficulty draws. Across benchmarks, the posterior probability that the observed $\mathrm{Max\ Gap}$ exceeds the same-size uniform baseline has median 0.999. The observed-to-baseline ratio has median $3.85\times$, and is much larger for several large benchmarks, e.g., $18.8\times$ on ScienceQA, $11.3\times$ on MedQA, and $5.9\times$ on ZebraLogic. Therefore, although the exact $\mathrm{Max\ Gap}$ magnitude carries posterior uncertainty, the conclusion that many benchmarks contain coverage gaps far beyond random uniform sampling is robust.

\subsection{Additional Experimental Details and Results}
\label{sec:app_additional_main}
Our code repository is available at:
\begin{center}
\small
\url{https://github.com/ASTRAL-Group/benchmark-difficulty-measure}.
\end{center}
We conducted all experiments on NVIDIA L40S GPUs with 48GB memory and NVIDIA RTX PRO 6000 Blackwell Server Edition GPUs with 96GB memory. Fig.~\ref{fig:appendix_result} provides the difficulty-distribution plots for the remaining benchmarks not shown in Fig.~\ref{fig:main_result}, following the same visualization format.

\begin{figure*}[t]
    \centering
    \includegraphics[width=0.95\textwidth,height=0.8\textheight]{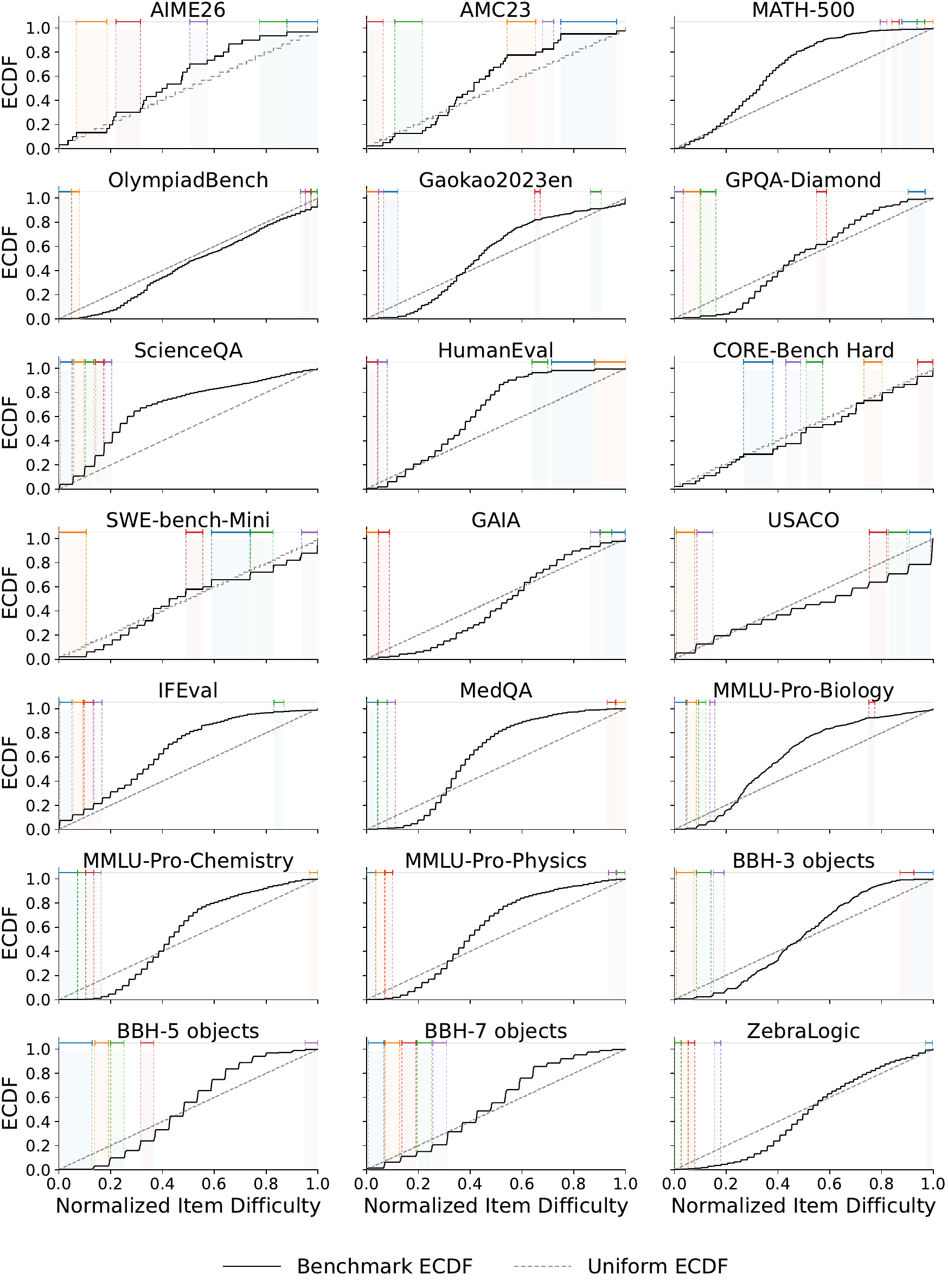}
    \caption{\textbf{Additional benchmark difficulty distributions.}
    Each panel shows the ECDF of normalized item difficulties for one benchmark. The black curve denotes the benchmark ECDF, and the gray dashed line denotes the uniform ECDF, which we use as the target difficulty distribution. Deviations from the uniform ECDF indicate mismatch relative to this neutral target distribution. Colored shaded intervals mark the five largest adjacent gaps between sorted item difficulties.}
    \label{fig:appendix_result}
\end{figure*}

\section{Benchmark Calibration}
\label{sec:app_benchmark_cal}
\paragraph{Problem setting.}
The measure-mismatch perspective suggests a natural way to redesign a benchmark: adjust the benchmark empirical measure so that it better matches a target evaluation distribution. Consider a benchmark
\[
B=\{x_1,\ldots,x_n\},
\]
where each item \(x_i\) has an estimated difficulty \(z_i\in[0,1]\). The benchmark induces the empirical measure
\[
\nu_B=\frac{1}{n}\sum_{i=1}^n \delta_{z_i}
\]
over the difficulty axis. Given a target distribution \(\mu\in\mathcal P([0,1])\), benchmark calibration aims to construct a new benchmark \(B'\) whose empirical measure \(\nu_{B'}\) is closer to \(\mu\) than \(\nu_B\) is.

In this work, we instantiate benchmark calibration through pruning. Specifically, we retain \(k\) items from the original benchmark by selecting a subset \(A\subseteq[n]\) with \(|A|=k\). The selected items form a calibrated benchmark
\[
B_A=\{x_i:i\in A\},
\]
with empirical difficulty measure
\[
\nu_A=\frac{1}{k}\sum_{i\in A}\delta_{z_i}.
\]
The goal is to choose \(A\) so that \(\nu_A\) better approximates the target distribution \(\mu\), for example as measured by \(\mathrm{DDM}(B_A;\mu)=W_1(\nu_A,\mu)\). 

\paragraph{A simple calibration algorithm.}
We use a simple greedy quantile-matching pruning procedure, summarized in Alg.~\ref{alg:benchmark-calibration}. Let \(Q_\mu\) denote the quantile function of the target distribution \(\mu\). Given a subset size \(k\), we first construct \(k\) target difficulty locations by taking the midpoints of \(k\) equal-probability intervals under \(\mu\):
\[
\tau_t=\frac{t-\frac12}{k},\qquad
q_t=Q_\mu(\tau_t),\qquad t=1,\ldots,k .
\]
We then match these target locations to existing benchmark items. Specifically, after normalizing item difficulties to \([0,1]\), we iterate over \(q_1,\ldots,q_k\) and select the unused item whose difficulty is closest to the current target location:
\[
i_t\in
\arg\min_{i\in [n]\setminus \{i_1,\ldots,i_{t-1}\}}
|z_i-q_t|.
\]
The calibrated subset is \(A=\{i_1,\ldots,i_k\}\). 

\begin{algorithm}[t]
\caption{Quantile-Matching Benchmark Calibration}
\label{alg:benchmark-calibration}
\begin{algorithmic}[1]
\Require Normalized item difficulties \(z_1,\ldots,z_n\in[0,1]\), subset size \(k\), target quantile function \(Q_\mu\)
\State Initialize selected set \(A\leftarrow\emptyset\) and available set \(U\leftarrow[n]\)
\For{\(t=1,\ldots,k\)}
    \State Set target quantile level
    \[
    \tau_t \leftarrow \frac{t-\frac12}{k}
    \]
    \State Set target difficulty location
    \[
    q_t \leftarrow Q_\mu(\tau_t)
    \]
    \State Select the closest unused item:
    \[
    i_t \in \arg\min_{i\in U}|z_i-q_t|
    \]
    \State Update \(A\leftarrow A\cup\{i_t\}\)
    \State Update \(U\leftarrow U\setminus\{i_t\}\)
\EndFor
\State \Return \(A\)
\end{algorithmic}
\end{algorithm}

\subsection{Experimental Setup}

In practice, given any target distribution over the difficulty axis, our benchmark calibration idea and Alg.~\ref{alg:benchmark-calibration} can be applied directly. In this experiment, we study whether such calibration can make a single benchmark produce model rankings that are more consistent with a broader in-domain benchmark family.

We consider \(N=8\) math benchmarks, as
 shown in Tab.~\ref{tab:benchmark_summary}. For each benchmark \(B_i\), we treat it as the source benchmark to be calibrated, and use the remaining benchmarks \(\{B_j:j\neq i\}\) only to construct an external reference ranking over models. This leave-one-out design evaluates whether the ranking induced by a calibrated version of \(B_i\) better agrees with the ranking obtained from other benchmarks in the same domain.

Concretely, let
\[
B_i=\{x_{i1},\ldots,x_{in_i}\}
\]
be the \(i\)-th benchmark, where item \(x_{i\ell}\) has normalized difficulty \(z_{i\ell}\in[0,1]\). Since the leave-one-out reference ranking is aggregated from multiple benchmarks in the same domain, it is less tied to the idiosyncratic difficulty concentration of any single source benchmark. We therefore use calibration to prune items from difficulty regions where the source benchmark is overly concentrated. In the absence of an application-specific target distribution, we instantiate the calibration target as the uniform distribution
\[
\lambda=\mathrm{Unif}[0,1].
\]
This choice should not be interpreted as assuming that the leave-one-out reference distribution is exactly uniform. Rather, uniform calibration provides a neutral debiasing strategy: it does not privilege any difficulty region a priori, and it counteracts benchmarks whose empirical mass is concentrated in a narrow portion of the difficulty axis.

We apply Alg.~\ref{alg:benchmark-calibration} to \(B_i\) with target distribution \(\lambda\), producing a calibrated subset \(A_i\subseteq[n_i]\) of size \(k_i\). The calibrated benchmark is
\[
B_{i,A_i}=\{x_{i\ell}:\ell\in A_i\}.
\]

Suppose all benchmarks evaluate the same \(M\) models. Let \(S_i(m)\) denote the average score of model \(m\) on the full source benchmark \(B_i\), and let \(S_i(A_i,m)\) denote its average score on the calibrated subset. The uncalibrated ranking is induced by
\[
S_i(m),\qquad m=1,\ldots,M,
\]
while the calibrated ranking is induced by
\[
S_i(A_i,m),\qquad m=1,\ldots,M.
\]

For the leave-one-out reference ranking, we aggregate model scores over the remaining benchmarks:
\[
S^{\mathrm{ref}}_{-i}(m)
=
\frac{1}{N-1}\sum_{j\neq i} S_j(m).
\]
The ranking induced by \(S^{\mathrm{ref}}_{-i}(m)\) serves as the external reference ranking for source benchmark \(B_i\).

For each source benchmark, we compute the Spearman and Kendall correlations between the source-induced ranking and the leave-one-out reference ranking, both before and after calibration. Improved correlation after calibration suggests that pruning the source benchmark toward a more balanced difficulty distribution makes its model comparison more consistent with the broader in-domain benchmark family.

\begin{table*}[t]
\centering
\small
\caption{
\textbf{Benchmark calibration results on eight math benchmarks.}
For each source benchmark, \textbf{Full} denotes the original benchmark,
\textbf{Random} denotes the average over 30 random pruned subsets of the same size,
and \textbf{Cal.} denotes the calibrated subset produced by our quantile-matching benchmark calibration algorithm.
Spearman and Kendall measure agreement with the leave-one-out reference ranking.
We report \(W_1(\nu,\mathrm{Unif})\) as a mechanism metric to verify that calibration moves the selected benchmark measure closer to the neutral uniform distribution. Bold values indicate the best result among \textbf{Full}, \textbf{Random}, and \textbf{Cal.} for each benchmark and metric.
}
\label{tab:benchmark_calibration}
\begin{adjustbox}{width=\textwidth}
\begin{tabular}{lccccccccc}
\toprule
& \multicolumn{3}{c}{\textbf{Spearman $\uparrow$}}
& \multicolumn{3}{c}{\textbf{Kendall $\uparrow$}}
& \multicolumn{3}{c}{\(\boldsymbol{W_1(\nu,\mathrm{Unif})}\) \(\boldsymbol{\downarrow}\)} \\
\cmidrule(lr){2-4}
\cmidrule(lr){5-7}
\cmidrule(lr){8-10}
\textbf{Benchmark}
& \textbf{Full} & \textbf{Random} & \textbf{Cal.}
& \textbf{Full} & \textbf{Random} & \textbf{Cal.}
& \textbf{Full} & \textbf{Random} & \textbf{Cal.} \\
\midrule
Minerva-Math  & 0.601 & 0.649 & \textbf{0.790} & 0.455 & 0.514 & \textbf{0.667} & 0.122 & 0.115 & \textbf{0.005} \\
MATH-500      & 0.930 & 0.942 & \textbf{0.958} & 0.879 & 0.878 & \textbf{0.909} & 0.164 & 0.165 & \textbf{0.099} \\
Gaokao2023en  & 0.888 & 0.895 & \textbf{0.951} & 0.788 & 0.790 & \textbf{0.879} & 0.088 & 0.086 & \textbf{0.029} \\
AIME24        & \textbf{0.944} & 0.942 & \textbf{0.944} & \textbf{0.848} & 0.847 & \textbf{0.848} & 0.077 & 0.108 & \textbf{0.019} \\
AIME25        & 0.958 & 0.936 & \textbf{0.972} & 0.909 & 0.854 & \textbf{0.915} & 0.079 & 0.126 & \textbf{0.020} \\
AIME26        & 0.949 & 0.950 & \textbf{0.970} & 0.870 & 0.867 & \textbf{0.901} & 0.085 & 0.108 & \textbf{0.017} \\
AMC23         & \textbf{0.937} & 0.934 & \textbf{0.937} & \textbf{0.848} & \textbf{0.848} & \textbf{0.848} & 0.089 & 0.093 & \textbf{0.025} \\
OlympiadBench & \textbf{0.916} & 0.910 & \textbf{0.916} & \textbf{0.818} & 0.806 & \textbf{0.818} & 0.064 & 0.067 & \textbf{0.008} \\
\midrule
Average       & 0.890 & 0.895 & \textbf{0.930} & 0.802 & 0.801 & \textbf{0.848} & 0.096 & 0.109 & \textbf{0.028} \\
\bottomrule
\end{tabular}
\end{adjustbox}
\end{table*}
\subsection{Results}

Experimental results are reported in Tab.~\ref{tab:benchmark_calibration}. The results show that our calibration algorithm improves the agreement between the source benchmark ranking and the leave-one-out reference ranking. On average, calibration improves Spearman correlation from \(0.890\) to \(0.930\), and Kendall correlation from \(0.802\) to \(0.848\). The improvement is especially pronounced on benchmarks with weaker initial agreement. Notably, on Minerva-Math, the full benchmark achieves only \(0.601\) Spearman and \(0.455\) Kendall correlation with the reference ranking, while calibration improves them to \(0.790\) and \(0.667\), respectively. Similarly, on Gaokao2023en, calibration improves Spearman from \(0.888\) to \(0.951\) and Kendall from \(0.788\) to \(0.879\). For the remaining benchmarks, the initial agreement is already high, leaving less room for improvement; nevertheless, calibration consistently preserves this high level of agreement while using the pruned benchmark.

We include a random pruning baseline that uses the same pruning size as our calibrated subset and average its results over 30 random trials. Random pruning is less effective than calibration: it achieves \(0.895\) average Spearman and \(0.801\) average Kendall correlation, both below the calibrated subset. We also report \(W_1(\nu,\mathrm{Unif})\) as a mechanism metric. Calibration consistently reduces this distance across all benchmarks, decreasing the average value from \(0.096\) to \(0.028\). In contrast, random pruning does not reduce this distance on average, yielding an average value of \(0.109\). This suggests that the improvement is not simply due to using a smaller benchmark, but comes from selecting items whose empirical difficulty measure better matches the neutral calibration target.

\section{Additional Details and Experimental Results of \textsc{BenchFill}}

\newpage
\begin{algorithm}[t]
\caption{\textsc{BenchFill}}
\label{alg:gap_filling}
\begin{algorithmic}[1]
\Require Benchmark items with normalized difficulties $\{(x_i,z_i)\}_{i=1}^n$; target gap $[z_L,z_R]$; rewriter LLM $\mathcal{G}$; validity judge $\mathcal{J}$; IRT evaluation LLMs $\{M_j\}_{j=1}^m$ with pre-estimated IRT abilities $\{\hat{\theta}_j\}_{j=1}^m$; trials per LLM $\{K_j\}_{j=1}^m$; number of parents $R$; children per parent $C$; reference list budget $L$; maximum rounds $T$; tolerance $\tau$ on the normalized $[0,1]$ scale.
\Ensure A proposed gap-filling item $x^\star_{\mathrm{item}}$.

\State Set target difficulty $z^\star \gets (z_L+z_R)/2$.
\State Initialize memory $\mathcal{M}$ with existing annotated item--difficulty pairs near the selected gap.
\State Initialize valid candidate set $\mathcal{C}_{\mathrm{valid}} \gets \emptyset$.

\For{$t=1,\dots,T$}
    \State Retrieve a parent set $\mathcal{S}_t \subseteq \mathcal{M}$ with at most $R$ items whose annotated difficulties are closest to $z^\star$.
    \State Construct a reference list $\mathcal{R}_t \subseteq \mathcal{M}$ with $|\mathcal{R}_t|\le L$ using nearby item--difficulty pairs around $z^\star$, with up to half selected from below $z^\star$ and the rest from above $z^\star$ when available.
    
    \For{each parent item $(x_p,z_p) \in \mathcal{S}_t$}
        \State Rescale the displayed difficulties in $z^\star$ and $\mathcal{R}_t$ from $[0,1]$ to $[0,10]$ for prompting.
        \State Rewrite the parent into $C$ child items conditioned on the parent item, target difficulty, and reference list:
        \[
        \{x_{p,1},\dots,x_{p,C}\}
        \gets
        \mathcal{G}(x_p, z^\star, \mathcal{R}_t).
        \]
        
        \For{each child item $x_{p,c}$}
            \If{$\mathcal{J}$ verifies that the reference answer of $x_{p,c}$ is correct}
                \For{$j=1,\dots,m$}
                    \State Evaluate LLM $M_j$ on $x_{p,c}$ for $K_j$ trials and record $c_j$ correct responses.
                \EndFor
                \State Estimate the item difficulty by online calibration:
                \[
                \hat{b}_{p,c}
                \gets
                \arg\max_b
                \sum_{j=1}^{m}
                \left[
                c_j\log \sigma(\hat{\theta}_j-b)
                +
                (K_j-c_j)\log\bigl(1-\sigma(\hat{\theta}_j-b)\bigr)
                \right].
                \]
                \State Normalize $\hat{b}_{p,c}$ using the same transformation as the original benchmark and obtain $\hat{z}_{p,c}$.
                \State Add $(x_{p,c},\hat{z}_{p,c})$ to $\mathcal{M}$.
                \State Add $(x_{p,c},\hat{z}_{p,c})$ to $\mathcal{C}_{\mathrm{valid}}$.
            \EndIf
        \EndFor
    \EndFor
    
    \If{$\mathcal{C}_{\mathrm{valid}} \neq \emptyset$ and
    $\min_{(x,\hat{z})\in \mathcal{C}_{\mathrm{valid}}} |\hat{z}-z^\star| \le \tau$}
        \State \textbf{break}
    \EndIf
\EndFor

\If{$\mathcal{C}_{\mathrm{valid}} = \emptyset$}
    \State \Return failure
\EndIf

\State Select the closest valid child item:
\[
(x^\star_{\mathrm{item}},\hat{z}^\star)
\gets
\arg\min_{(x,\hat{z})\in \mathcal{C}_{\mathrm{valid}}}
|\hat{z}-z^\star|.
\]
\State \Return $x^\star_{\mathrm{item}}$ as the proposed gap-filling item, subject to human expert review.
\end{algorithmic}
\end{algorithm}

\subsection{Implementation Details}
\label{sec:app_gap}

For each benchmark, we focus on the three most pronounced gaps in its estimated difficulty distribution and rewrite one item for each gap. Since an LLM may not reliably distinguish small numerical differences on the normalized $[0,1]$ difficulty scale when the target difficulty is provided in the prompt, we rescale difficulty values to $[0,10]$ for prompting while keeping all reported metrics on the original normalized scale.

For the reference list, we include nearby item--difficulty pairs on both sides of the target difficulty. Specifically, half of the reference items are selected from those with difficulty below the target and closest to it, and the other half are selected from those with difficulty above the target and closest to it. This provides the rewriter with local bidirectional context around the desired difficulty level. The full rewriter prompt template is provided below.

For difficulty annotation, we use the same set of LLMs and the same generation configuration as in the original difficulty estimation stage. The search stops early if the best valid item found so far is within tolerance $\tau=0.01$ of the target difficulty on the normalized $[0,1]$ scale, or otherwise stops after the maximum number of rounds. We then select the valid item whose annotated difficulty is closest to the target. Before being inserted into the benchmark, each selected item undergoes a manual validity check by at least three human experts, who verify its correctness.

We additionally implement a budget-matched best-of-\(N\) uniform-parent rewriting baseline. For each target gap \(g\), let \(T_g\) denote the number of rounds actually run by \textsc{BenchFill} before stopping. Since \textsc{BenchFill} rewrites \(4\) parents per round and produces \(4\) child items per parent, its proposal budget for this gap is \(T_g \times 4 \times 4 = 16T_g\) child items. To match this budget, the baseline uniformly samples \(4\) existing benchmark items as parents and asks the same rewriter LLM to rewrite \(4T_g\) child items from each parent, conditioned on the parent item and the target difficulty. After difficulty annotation, it selects the valid child item whose difficulty is closest to the target among all generated candidates. Thus, the baseline uses the same total rewrite budget and the same final selection rule as \textsc{BenchFill}, but does not use difficulty-aware parent retrieval, nearby reference lists, or iterative memory updates.

\begin{center}
\begin{tcolorbox}[
    title={Rewriter Prompt Template},
    label={box:proposer_prompt},
    width=\linewidth,
    colback=white,
    colframe=black,
    boxrule=0.5pt,
    left=6pt,
    right=6pt,
    top=6pt,
    bottom=6pt,
    fonttitle=\bfseries,
    fontupper=\ttfamily\small,
]

Rewrite the parent item below into \textbf{exactly 1} new benchmark item variant.
Target difficulty (0--10 scale, higher = harder):
\textbf{\{target\_difficulty\}}.

\medskip
\textbf{Parent item} (rewrite this item into a new variant -- do not copy its specific wording, numbers, or answer):

\{parent.question\}

Answer: \{parent.reference\_answer\}

Difficulty (0--10): \{parent.difficulty\} $|$ delta \{target - parent.difficulty\} $|$ generate \texttt{\{harder|easier\}}

\medskip
\textbf{Memory items} (use as style/difficulty reference only, do not copy):

\medskip
--- Memory 1 ---

\{memory[0].question\}

Answer: \{memory[0].reference\_answer\}

Difficulty (0--10): \{memory[0].difficulty\} $|$ delta \{target - memory[0].difficulty\} $|$ generate \texttt{\{harder|easier\}}

\medskip
--- Memory 2 ---

\{memory[1].question\}

Answer: \{memory[1].reference\_answer\}

Difficulty (0--10): \{memory[1].difficulty\} $|$ delta \{target - memory[1].difficulty\} $|$ generate \texttt{\{harder|easier\}}

\medskip
... (one block per memory item) ...

\medskip
Return \textbf{only} one markdown code block: start the first line with \texttt{\textasciigrave\textasciigrave\textasciigrave{}json}, then a newline, then a single JSON array with \textbf{exactly 1} element, then close with \texttt{\textasciigrave\textasciigrave\textasciigrave} on its own line. Do not put anything after the closing fence.

\medskip
Do not stop until the JSON is \textbf{syntactically complete}: every string value is closed with \texttt{"}, every object is closed with \texttt{\}}, the array ends with \texttt{]}, and then the closing \texttt{\textasciigrave\textasciigrave\textasciigrave} line.

\medskip
The array must have \textbf{exactly 1} element -- an object with exactly these keys:
\texttt{"id"} (string), \texttt{"question"} (string), \texttt{"answer"} (string).

\end{tcolorbox}
\end{center}

\subsection{Additional Experimental Results}



\paragraph{Human validation.}
We ensure the reliability of the generated gap-filling items using a two-stage validation process. GPT-5.4 first serves as an automatic validity judge, checking whether each candidate is well-formed and whether its reference answer is correct; candidates that fail this check are discarded and regenerated. Human experts then independently audit the final selected items. Across the 24 target gaps, one initially proposed candidate failed the GPT-5.4 check because of an incorrect reference answer and was regenerated. All 24 final selected items passed the GPT-5.4 validity check and subsequently passed human review without modification or replacement.

\paragraph{Does the rewriter perceive difficulty?} A natural question is whether our rewriter indeed perceives difficulty, instead of hitting the target difficulty blindly. We examine this question from two complementary angles: whether the rewriter already shows difficulty sensitivity before iterative feedback, and whether the subsequent search-and-annotate loop further improves alignment with the target difficulty.

First, to isolate the rewriter's own difficulty perception from later IRT feedback and iterative search, we analyze only the first generation round of each gap-filling case. At this stage, the rewriter is conditioned on the specified target difficulty, but has not yet received feedback on the annotated difficulties of its generated children. Fig.~\ref{fig:first_round_calibration} shows the specified target difficulty and the annotated difficulty of valid first-round children across all 8 benchmarks and 3 target gaps per benchmark. Within most benchmarks, higher specified targets produce higher mean annotated difficulties, and well-separated targets usually induce distinct child-difficulty distributions. This indicates that the rewriter already has non-trivial sensitivity to the specified difficulty scale before iterative feedback is introduced.

Second, we examine whether the search-and-annotate loop further improves target alignment. Fig.~\ref{fig:multi_round_deviation} shows one representative gap-filling case from each benchmark. Across these cases, the mean absolute target deviation generally decreases from the first round to the final round, and in most cases the generated children also become more concentrated around the target. This suggests that the iterative loop further refines the rewriter's initial difficulty signal using IRT annotation and updated memory.

\begin{figure*}[t]
    \centering
    \includegraphics[width=\textwidth]{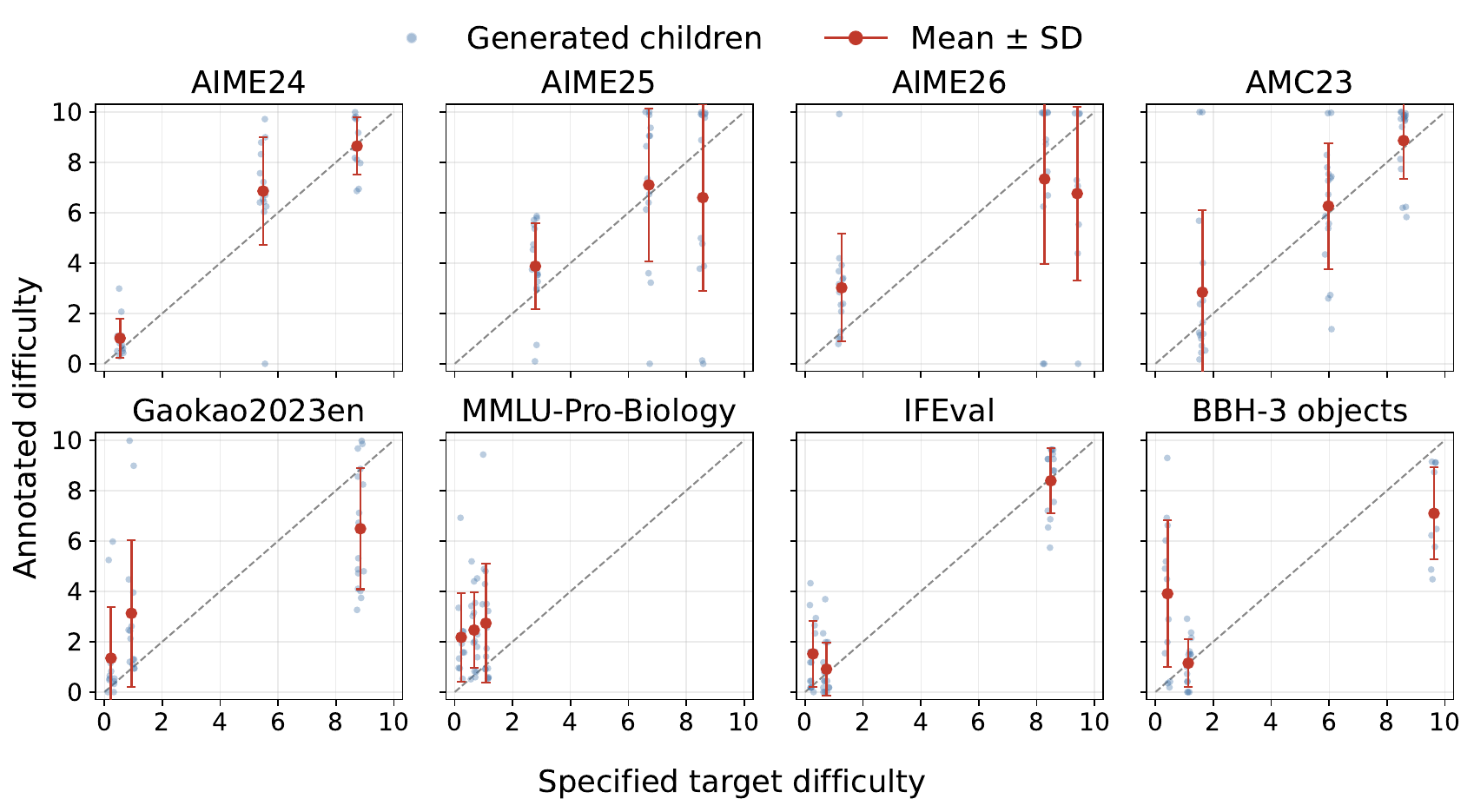}
    \caption{\textbf{First-round alignment of specified target difficulty and annotated difficulty.}
    Each panel corresponds to one benchmark. The horizontal axis shows the specified target difficulty, and the vertical axis shows the annotated difficulty of valid first-round children. Blue points denote individual generated children, and red markers denote the mean and standard deviation for each target difficulty. The dashed diagonal indicates perfect alignment. For readability, both axes are shown on a 0--10 scale, corresponding to normalized difficulties in $[0,1]$ multiplied by 10. Within most benchmarks, higher specified targets produce higher annotated difficulties already in the first round, before any iterative feedback is incorporated.}
    \label{fig:first_round_calibration}
\end{figure*}

\begin{figure*}[t]
    \centering
    \includegraphics[width=\textwidth]{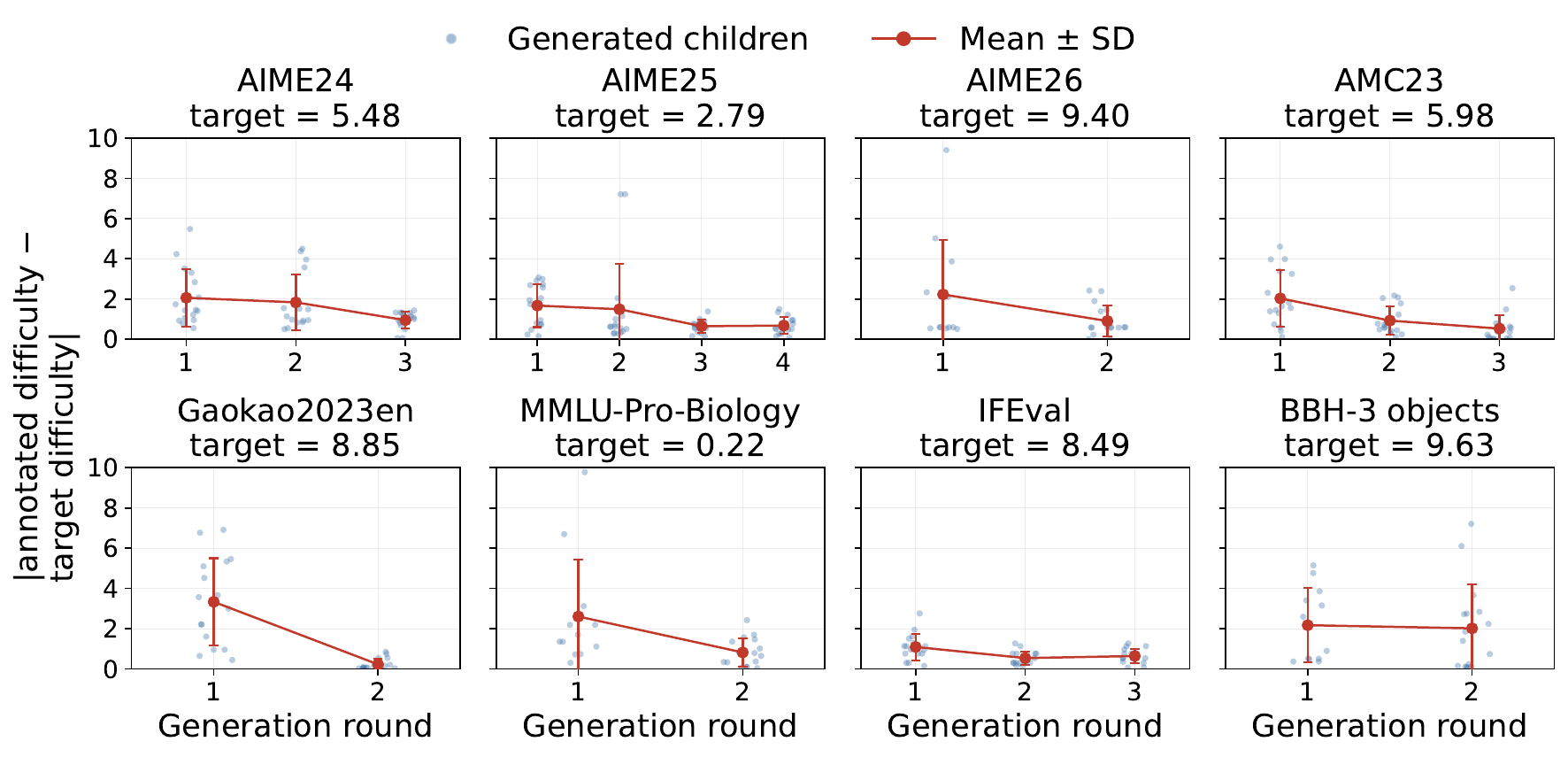}
    \caption{\textbf{Multi-round target alignment during \textsc{BenchFill} search.}
    Each panel shows one representative gap-filling case from a different benchmark. The horizontal axis denotes the generation round, and the vertical axis denotes the absolute target deviation, $|\hat{z} - z^\star|$, for generated children. For readability, target difficulties and absolute deviations are multiplied by 10. Blue points denote individual generated children, and red markers denote the round-level mean and standard deviation. 
    Across most cases, the mean target deviation decreases over rounds, and the generated children become more concentrated around the target, showing that the search-and-annotate loop further improves target alignment beyond the rewriter's initial difficulty sensitivity.}
    \label{fig:multi_round_deviation}
\end{figure*}

\paragraph{Detecting localized model improvements.}
We further examine whether \textsc{BenchFill} improves a benchmark's ability to detect localized model improvements beyond reducing difficulty gaps. We consider one target gap in each of AIME24, AIME25, and AIME26, and compare three benchmark variants: \texttt{Original}, \texttt{Original + BenchFill}, and \texttt{Original + Baseline}. Both augmentation methods insert exactly one item into the original benchmark. We perform mild SFT so that the model improves around the target gap region. Each item is evaluated with 32 rollouts before and after SFT, and we test whether the gain is statistically detectable using a one-sided stratified randomization test ($p<0.05$).

As shown in Tab.~\ref{tab:localized_detection}, \texttt{Original + BenchFill} detects a positive post-SFT gain in all three cases (3/3), while \texttt{Original} and \texttt{Original + Baseline} detect none (0/3 each). This shows that \textsc{BenchFill} is not merely improving the metric it optimizes; by filling under-covered difficulty regions, it makes the benchmark more sensitive to localized model improvements that would otherwise be missed.

\begin{table*}[t]
\centering
\caption{\textbf{Localized-improvement detection results.}
Each augmentation method adds one item to the original 30-item benchmark.
Gain denotes the change in average accuracy.
The one-sided stratified randomization test randomly reassigns rollout outcomes
between the pre- and post-SFT conditions within each item using 100,000 permutations.
Boldface indicates variants for which a positive gain is detected at $p<0.05$.}
\label{tab:localized_detection}
\small
\setlength{\tabcolsep}{5pt}
\begin{tabular}{llccc}
\toprule
\textbf{Case}
& \textbf{Benchmark variant}
& \textbf{Gain}
& \textbf{One-sided $p$}
& \textbf{Detected} \\
\midrule
\multirow{3}{*}{\shortstack[l]{AIME24\\Olmo-3.1-32B-Think}}
& \texttt{Original}
& +0.104 & 0.500100 & $\times$ \\
& \textbf{\texttt{Original + BenchFill}}
& \textbf{+2.419} & \textbf{0.005940} & $\checkmark$ \\
& \texttt{Original + Baseline}
& +0.202 & 0.458000 & $\times$ \\
\midrule
\multirow{3}{*}{\shortstack[l]{AIME25\\Phi-4}}
& \texttt{Original}
& +1.042 & 0.184100 & $\times$ \\
& \textbf{\texttt{Original + BenchFill}}
& \textbf{+3.730} & \textbf{0.000450} & $\checkmark$ \\
& \texttt{Original + Baseline}
& +0.907 & 0.214200 & $\times$ \\
\midrule
\multirow{3}{*}{\shortstack[l]{AIME26\\Llama-3.1-8B-Instruct}}
& \texttt{Original}
& +0.208 & 0.418400 & $\times$ \\
& \textbf{\texttt{Original + BenchFill}}
& \textbf{+3.427} & \textbf{0.000010} & $\checkmark$ \\
& \texttt{Original + Baseline}
& +0.101 & 0.498000 & $\times$ \\
\bottomrule
\end{tabular}
\end{table*}

\paragraph{Representative examples.}
To provide a qualitative view of the generated gap-filling items, we show representative rewriting examples in Tab.~\ref{tab:gap_filling_examples}. Each example presents a final selected item together with its original ancestor from the benchmark, as well as the ancestor difficulty, target difficulty, and annotated difficulty of the selected item. The examples illustrate how \textsc{BenchFill} modifies existing benchmark items to reach previously under-covered difficulty levels.

\begin{table*}[t]
\centering
\footnotesize
\setlength{\tabcolsep}{2.5pt}
\renewcommand{\arraystretch}{1.15}
\caption{
\textbf{Representative pairs of original benchmark ancestors and final selected items produced by \textsc{BenchFill}.}
All difficulty values are normalized to $[0,1]$.
The main change from the ancestor item to the selected item is highlighted. For selected items, values in parentheses denote the target difficulty.
}
\label{tab:gap_filling_examples}
\begin{tabularx}{\textwidth}{
>{\centering\arraybackslash}p{0.125\textwidth}
>{\hspace{3pt}\raggedright\arraybackslash}p{0.275\textwidth}
>{\centering\arraybackslash}p{0.070\textwidth}
>{\raggedright\arraybackslash}p{0.395\textwidth}
>{\centering\arraybackslash}p{0.090\textwidth}
}
\toprule
\textbf{Benchmark}
& \multicolumn{2}{c}{\textbf{Original benchmark ancestor}}
& \multicolumn{2}{c}{\textbf{Selected gap-filling item}} \\
\cmidrule(lr){2-3}
\cmidrule(lr){4-5}
& \multicolumn{1}{c}{\textbf{Item}}
& \textbf{Diff.}
& \multicolumn{1}{c}{\textbf{Item}}
& \shortstack{\textbf{Diff.}} \\
\midrule

\textbf{AMC23}
&
For how many integers $n$ does the expression $\sqrt{\frac{\log (n^2) - (\log n)^2}{\log n - \colorbox{red!20}{3}}}$ represent a real number, where $\log$ denotes the base-$10$ logarithm?
&
$0.535$
&
Find the number of integer values of $n$ such that $\sqrt{\frac{\log (n^2) - (\log n)^2}{\log n - \colorbox{red!20}{1}}}$ represents a real number, where $\log$ denotes the base-$10$ logarithm.
&
$0.597$ $(0.598)$
\\

\midrule

\textbf{Gaokao2023en}
& On an \colorbox{red!20}{$8 \times 8$ chessboard,} \colorbox{red!20}{6 black rooks} and $k$ white rooks are placed on different cells so that each rook only attacks rooks of the opposite color. Compute the maximum possible value of $k$. (Two rooks attack each other if they are in the same row or column and no rooks are between them.)
&
$0.864$
& On a \colorbox{red!20}{$9 \times 9$ chessboard,} \colorbox{red!20}{7 black rooks} and $k$ white rooks are placed on distinct squares so that each rook attacks only rooks of the opposite color. What is the maximum possible value of $k$? (Two rooks attack each other if they are in the same row or column and there is no rook between them.)
&
$0.886$ $(0.885)$
\\

\midrule

\textbf{AIME25}
&
Alex divides a disk into \colorbox{red!20}{four quadrants} with two perpendicular diameters intersecting at the center of the disk. He draws \colorbox{red!20}{$25$ more line segments} through the disk, each connecting two points at random on the perimeter in \colorbox{red!20}{different quadrants}. Find the expected number of regions into which these \colorbox{red!20}{$27$ line segments} divide the disk.
&
$0.750$
&
A circle has a diameter drawn, splitting it into \colorbox{red!20}{two semicircular regions}. Then \colorbox{red!20}{8 chords} are drawn, each connecting a random point on the boundary in \colorbox{red!20}{one semicircle} to a random point on the boundary in \colorbox{red!20}{the other semicircle}. Assuming all points are chosen uniformly and independently, what is the expected number of regions into which the interior of the circle is divided?
&
$0.674$ $(0.671)$
\\

\bottomrule
\end{tabularx}
\end{table*}


\newpage

\end{document}